\documentclass[11pt,prd,aps,amsfonts,eqsecnum,superscriptaddress,nofootinbib,longbibliography,notitlepage]{revtex4-1}

\usepackage{lmodern}    
\usepackage{microtype}  
\usepackage[english]{babel}
\usepackage{setspace}

\usepackage{fancyhdr}

\usepackage{caption}
\usepackage{graphicx}
\usepackage{xcolor}
\usepackage{subcaption}
\usepackage{rotating}
\usepackage{amsmath, amssymb, graphics, amsthm, isomath}
\usepackage{physics}
\usepackage{verbatim}
\usepackage{tikz}
\usepackage{tikz-cd}
\usepackage{circuitikz}
\usetikzlibrary{patterns}
\usetikzlibrary{through,hobby}
\usetikzlibrary{decorations.markings}
\usetikzlibrary{fadings,shadings}
\usetikzlibrary{plotmarks}
\usetikzlibrary{positioning}
\usetikzlibrary{decorations,arrows}
\usetikzlibrary{decorations.pathreplacing}
\usetikzlibrary{chains, scopes, positioning, backgrounds, shapes, fit, shadows, calc, arrows.meta, decorations.pathreplacing}
\usepackage{blkarray}
\usepackage{svg}
\usepackage{mathrsfs}
\usepackage{algorithm}
\usepackage{algpseudocode}

\usepackage{xcolor}
\definecolor{mycitecolor}{rgb}{0.0, 0.45, 0.85}   

\newcommand{\im}{\operatorname{im}}

\usepackage[
  colorlinks=true,
  linkcolor=mycitecolor,
  citecolor=mycitecolor,
  urlcolor=mycitecolor,
  hyperindex=true,
  linktocpage=true
]{hyperref}

\usepackage[capitalise,compress]{cleveref}
\usepackage{tcolorbox}

\newtcolorbox{shadedtheorem}{
  colback=gray!15,    
  colframe=white,     
  boxrule=0pt,        
  arc=0pt,            
  left=5pt,           
  right=5pt,
  top=5pt,
  bottom=5pt
}

\newtheoremstyle{upright-hang}
  {6pt}   
  {6pt}   
  {\normalfont\hangindent=1.5em\hangafter=1\relax} 
  {0pt}   
  {\bfseries} 
  {:}     
  {0.5em} 
  {}      

\theoremstyle{upright-hang}
\newtheorem{thm}{Theorem}
\numberwithin{thm}{section}
\newtheorem{cor}[thm]{Corollary}
\newtheorem{alg}[thm]{Algorithm}
\newtheorem{lem}[thm]{Lemma}

\newtheorem{prop}[thm]{Proposition}

\newtheorem{defn}[thm]{Definition}

\renewcommand{\thesection}{\arabic{section}}
\renewcommand{\thesubsection}{\thesection.\arabic{subsection}}
\renewcommand{\thesubsubsection}{\thesubsection.\arabic{subsubsection}}

\makeatletter
\renewcommand{\p@subsection}{}
\renewcommand{\p@subsubsection}{}
\makeatother

\makeatletter
\fancypagestyle{plain}{%
  \fancyhf{}%
  \fancyfoot[C]{\thepage}%
}
\fancypagestyle{revtexfooter}{%
  \fancyhf{}%
  \fancyfoot[C]{\thepage}%
}
\makeatother
\AtBeginDocument{\setstretch{1.25}\pagestyle{revtexfooter}}
\makeatletter
\renewcommand\frontmatter@title@format{\LARGE\bfseries\centering\parskip\z@skip}

\def\frontmatter@affiliationfont{\normalfont\selectfont}
\newcommand{\AndySectionStar}[1]{%
  \par\addpenalty\@secpenalty
  \addvspace{3.5ex \@plus 1ex \@minus .2ex}%
  \noindent{\normalfont\Large\bfseries #1\par}%
  \nobreak\vspace{2.3ex \@plus .2ex}%
  \@afterheading}
\newcommand{\AndySubsectionStar}[1]{%
  \par\addpenalty\@secpenalty
  \addvspace{3.25ex \@plus 1ex \@minus .2ex}%
  \noindent{\normalfont\large\bfseries #1\par}%
  \nobreak\vspace{1.5ex \@plus .2ex}%
  \@afterheading}
\newcommand{\AndySubsubsectionStar}[1]{%
  \par\addpenalty\@secpenalty
  \addvspace{3.25ex \@plus 1ex \@minus .2ex}%
  \noindent{\normalfont\normalsize\bfseries #1\par}%
  \nobreak\vspace{1.5ex \@plus .2ex}%
  \@afterheading}
\renewcommand\@seccntformat[1]{\csname the#1\endcsname\quad}
\def\@hangfrom@section#1#2#3{\@hangfrom{#1#2}#3}
\def\@hangfroms@section#1#2{#1#2}
\renewcommand\section{\@ifstar{\AndySectionStar}{\@startsection {section}{1}{\z@}%
  {-3.5ex \@plus -1ex \@minus -.2ex}%
  {2.3ex \@plus .2ex}%
  {\normalfont\Large\bfseries}}}
\renewcommand\subsection{\@ifstar{\AndySubsectionStar}{\@startsection{subsection}{2}{-\parindent}%
  {-3.25ex\@plus -1ex \@minus -.2ex}%
  {1.5ex \@plus .2ex}%
  {\normalfont\large\bfseries}}}
\renewcommand\subsubsection{\@ifstar{\AndySubsubsectionStar}{\@startsection{subsubsection}{3}{-\parindent}%
  {-3.25ex\@plus -1ex \@minus -.2ex}%
  {1.5ex \@plus .2ex}%
  {\normalfont\normalsize\bfseries}}}
\renewcommand*\l@section[2]{%
  \ifnum \c@tocdepth >\z@
    \addpenalty\@secpenalty
    \addvspace{0.75em \@plus\p@}%
    \begingroup
      \parindent\z@
      \rightskip\@pnumwidth
      \parfillskip-\@pnumwidth
      \leavevmode\bfseries
      \setlength\@tempdima{2.3em}%
      \advance\leftskip\@tempdima
      \hskip-\leftskip
      #1\nobreak\hfil\nobreak\hb@xt@\@pnumwidth{\hss #2}\par
    \endgroup
  \fi}
\renewcommand*\l@subsection{\@dottedtocline{2}{1.5em}{3.2em}}
\renewcommand*\l@subsubsection{\@dottedtocline{3}{3.8em}{4.1em}}
\makeatother

\usepackage{mathtools}
\tikzstyle{densely dashed}= [dash pattern=on 4pt off 3pt]

\usepackage{dsfont}

\allowdisplaybreaks

\begin{document}

\title{Sample-efficient learning with Dense Associative Memories}  

\author{Andrew Lucas}
\email{andrew.j.lucas@colorado.edu}
\affiliation{Department of Physics and Center for Theory of Quantum Matter, University of Colorado, Boulder, CO 80309, USA}

\author{Dmitry Krotov}
\email{krotov@dynamicalmind.com}
\affiliation{Dynamical Mind, Cambridge, MA 02141, USA}

\begin{abstract}
We analytically study the evolution of the energy function of Dense Associative Memories (DenseAMs) as they train on randomly sampled codewords of random expander low-density parity-check codes (a.k.a. solutions of random XORSAT problems).  We identify two distinct regimes and a transition between them, reminiscent of grokking in neural networks. A regime of small training set size is where the model memorizes a few training examples but cannot discover the underlying rules of the entire codeword dictionary generative process. The opposite regime, corresponding to a larger training set size, is where the model has uncovered the sparse constraints that generated the entire data set after seeing only polynomially many patterns. We also design a novel class of DenseAMs with a thresholding learning rule, and show that any amount of thresholding enables the model to perform simplified sparse Hamiltonian learning, which exponentially reduces the number of training examples needed to learn the true underlying data set. In this regime, the model learns how to discover $\exp[\mathrm{O}(n)]$ solutions from only $\mathrm{O}(\log n)$ training samples, which is asymptotically optimal for any learning algorithm. When training on codewords of certain random codes, the model can even learn the underlying code when a small O(1) fraction of bits are latent variables.  These latent variables are not shown to the model during learning, but will be discovered by sparsifying observed higher-order correlations.  For certain random codes, we show that a simple and conservative algorithm learns a minimal latent representation for the data set.  This provides a transparent, physics-motivated, and analytically tractable toy model for how deep latent representations emerge in an associative memory training.
\end{abstract}
\maketitle

\newpage
\tableofcontents
\newpage

\section{Introduction}
A memorization-to-generalization transition is a hallmark of modern AI models. At early stages of training, large models typically memorize individual instances of training data. As training progresses, the models discover more abstract features and transition from relying on example-specific details to forming reusable knowledge via internal representations. This phenomenon has received substantial attention in Large Language Models, where it is typically called ``grokking'' \cite{power2022grokking, gromov2023grokkingmodulararithmetic,MohamadiLiWuSutherland2024,morris2025much, li2026grokking}. Similar observations have been made in Diffusion Models \cite{somepalli2023diffusion,carlini2023extracting, pham2025memorization, bonnaire2026diffusion}, and Dense Associative Memories (DenseAMs) \cite{pham2025memorization}. Interest in this phenomenon is motivated both by practical concerns -- including privacy and copyright infringement arising from the memorization of training data -- and by a more fundamental question: how do individual training examples turn into generalizable knowledge encoded in a model's weights?

Although the notion of a memorization-to-generalization transition is intuitive and widely used, it remains difficult to characterize rigorously in the context of natural datasets. The underlying rules governing natural datasets are generally unknown. Consequently, there is no canonical criterion for determining when a model has transitioned from memorizing individual training examples to learning the underlying data-generating structure. This ambiguity also makes the transition difficult to quantify as a function of the training set size.

In this paper we study this problem in a simplified setting where precise results can be found.  Specifically, we study how DenseAMs learn random instances of the XORSAT problem; equivalently, how they learn low-density parity-check codes over binary variables. DenseAMs are a prominent class of energy-based models, which have super-linear (in the size of the neural network) memory storage capacity \cite{krotov2016}. They generalize the Hopfield network of associative memory, which has a linear memory storage capacity \cite{hopfield_1982}. The generic feature of this class of models is that they have a very large number of basins of attraction (possibly exponentially large in the size of the network), and, at the same time, each basin has a substantial volume. The latter property makes if possible to perform error correction by following a trajectory of the dynamical system minimizing the energy of DenseAM. To set up the problem, we sample a small number of random codewords from expander low-density parity-check (LDPC) codes \cite{SipserSpielman1996,urbanke} to form the training set, and ask how many samples are needed until the DenseAM learns the underlying parity-checks (constraints) of the code, rather than simply memorizing individual instances from the training set. In other words, how many training samples are needed for DenseAM to recover the true underlying generative process of the LDPC codewords?

An LDPC code is a standard method for protecting digital information from errors during transmission or storage. When transmitting a string of bits,  an LDPC code adds a finite fraction of redundant bits:  the role of these redundant bits is to act as parity-checks on the original bits that encoded the logical information.  The code is ``low-density" when each bit is involved in only a finite-number of (small) parity-checks, and in turn each small check is a constraint involving only a finite number of bits.  
If some bits are corrupted by noise, the receiver can use the pattern of satisfied and violated checks to infer which bits are most likely incorrect.   For random LDPC codes, this procedure works with high probability when a small  $\mathrm{O}(1)$ fraction of the bits are corrupted.

The beauty of LDPC codes is that their codewords form a nice \emph{algorithmic dataset} --  every codeword comes from a clearly defined mathematical process: in this case, finding solutions to linear equations over $\mathbb{F}_2$, a.k.a. the integers (mod 2). This creates an opportunity for training an AI model on a small subset of codewords (the training set), and test the model on the remaining codewords (the test set) without ever revealing the underlying mathematical process of code word generation.   Obviously, we can test whether the AI has learned the underlying LDPC code by checking whether or not all of the unseen codewords are verified as solutions by the AI.  

As we will show, DenseAM is a perfect model to train on random LDPC codewords. The model is defined by an energy function, which has a huge (possibly exponential in the size of the network) number of local minima. Inference dynamics corresponds to a temporal evolution of a dynamical system decreasing the energy, and the low-energy states correspond to plausible data-points. The appealing aspect of DenseAMs is that it is possible to train them both via a self-supervised learning via the backpropagation algorithm \cite{krotov2018dense}, and, at the same time, via simpler Hebbian-style learning rules \cite{krotov2016}. The latter eliminates the complexity of training process from the main phenomenon that we focus on in this work: how sample complexity of the data drives the memorization-to-generalization transition. For this reason we adopt DenseAMs as a toy model laboratory for studying this phenomenon. 

In fact, the architecture of off-the-shelf DenseAM is more or less perfectly suited for the efficient learning of XORSAT.  The reason is that under Hebbian learning, the real parity-checks of the code correspond to weights in the energy function that are maintained as large during training, while false parity-checks are equally likely to be obeyed or violated.  Therefore under Hebbian learning, DenseAM learns to reduce the weight on false parity-checks until eventually, the energy landscape is dominated by the true parity-checks.  Far past this point (with increasing training set size), the inference dynamics described in the previous paragraph amounts to gradient descent in the \emph{underlying Hamiltonian}, where each codeword is a local minimum with a large energy basin around it.  Even better, we may read off the underlying rules by identifying the large weights in the learned energy function.   We can, therefore, quantitatively assess exactly how long it takes for DenseAM to learn the underlying code.

Our results are foreshadowed by earlier literature.  For example, a long time ago, Valiant proposed a PAC theory of approximately-learnable problems, with satisfiability problems (which include XORSAT) as a prime application \cite{valiant}.  More recently, a systematic theory of sparse Hamiltonian learning \cite{RavikumarWainwrightLafferty2010,Bresler2015,VuffrayMisraLokhovChertkov2016,KlivansMeka2017} has been developed, where it is a well-known result that only $\mathrm{O}(\log n)$ samples are necessary to learn a Hamiltonian when observing random samples drawn from its Gibbs state.    The Hamiltonian itself will typically have $\mathrm{poly}(n)$ terms, and its Gibbs state may cluster around $\exp[\mathrm{O}(n)]$ local minima, as in a typical spin glass or complex system \cite{MezardMontanari2009}; these features do not make learning the Hamiltonian itself parametrically harder! The key point is that sparse learning is capable of using the same samples to independently discover many different low-order correlations corresponding to different interactions in the Hamiltonian.  

The difference between sparse learning and the Hebbian learning rule of DenseAM is that the latter is a computationally efficient algorithm, without an all-purpose guarantee that it learns the correct problem.  Of particular note is that sparse learning is increasingly slow when the temperature of the sampled Gibbs state is very low.   It turns out, however, that XORSAT is a very special ``sparse Hamiltonian" to learn -- even at zero temperature (the limit of interest for our algorithmic dataset) we may still exactly learn the Hamiltonian.  To achieve the optimal $\mathrm{O}(\log n)$ scaling of the training set size predicted by sparse learning, we will add a thresholding rule to Hebbian learning, where suitably small couplings are, roughly, rounded down to zero.  We will prove that this modification to DenseAM causes an exponential increase in the efficiency of learning.  Related ideas appeared in \cite{DeDiakonikolasServedio2012,DeDiakonikolasServedio2015}, where the authors wished to mock the distribution of solutions to a satisfiability problem, rather than learn it exactly.

With or without thresholding, we can quantitatively describe the resulting memorization-to-generalization transition undergone by DenseAM as a function of training set size.  In the limit of high thresholding, there is a transition between memorization and generalization, reminiscent of the reverse code engineering problem from classical LDPC codes \cite{tillich,sudan}.  We prove that this transition is surprisingly sharp: with a sufficiently small $\mathrm{O}(\log n)$ number of training examples, DenseAM will literally only be able to memorize the seen examples (and their linear combinations); with only an O(1) factor of extra patterns, it learns the entire code.  In the limit of no thresholding, there is a simple transition as well.  After seeing few patterns, the energy landscape of DenseAM looks Hopfield-like \cite{krotov2016,hopfield_1982} and is dominated by the observed patterns.  After seeing sufficiently many [$\mathrm{poly}(n)$] patterns, the weights of the true parity-checks in the Hamiltonian dominate over the random noisy couplings dominated by the training set, and correspondingly the energy landscape increasingly matches XORSAT (as the learned energy function becomes the XORSAT Hamiltonian exactly).

Finally, we show that this optimal performance of DenseAM in learning XORSAT persists when a small fraction of the bits in XORSAT are not shown to the AI -- they are latent.  Again, the simplifying structure of XORSAT makes it possible for us to very efficiently learn the underlying latent variables, analogous to sparse latent learning \cite{BreslerKoehlerMoitra2019,BreslerBuhai2020}.

\section{Basics of Dense Associative Memory}
We are studying models where the patterns to learn are bitstrings of length $n$.  We denote each pattern as an $n$-component vector $(\xi^\mu_1,\ldots, \xi^\mu_n)$ where each $\xi^\mu_i \in \lbrace \pm 1\rbrace$.  Here and in what follows, the indices $i,j,\ldots \in \lbrace 1,\ldots, n\rbrace $ denote the bits, while $\mu,\nu,\ldots$ denote the patterns that we will feed the architecture.   We will choose patterns $\xi^\mu \in \mathcal{C}$ randomly out of some subset $\mathcal{C}\subseteq \lbrace \pm1\rbrace^n$.  Our goal will be to learn an energy function $\widehat{E}:\lbrace \pm 1\rbrace^N \rightarrow \mathbb{R}$,  such that every element of $\mathcal{C}$ is a local minimum of $\widehat{E}$.

A simple proposal is to choose $\widehat{E}$ as an $m$-local DenseAM \cite{krotov2016}:
\begin{equation}
    \widehat{E}(\sigma) = -\frac{1}{P}\sum_{\mu=1}^P \left(\sum_{i=1}^n \xi_i^\mu \sigma_i\right)^m. \label{eq:DAM}
\end{equation}
It is convenient to introduce tensor  
\begin{equation}
    T_{i_1\cdots i_m}:= \frac{1}{P}
\sum_{\mu=1}^P \xi^\mu_{i_1}\cdots \xi_{i_m}^\mu.  \label{eq:Ti1im}
\end{equation}
and its \emph{thresholded} version
\begin{equation}
    \widehat{T}_{i_1\cdots i_m} = \phi_a (T_{i_1\cdots i_m}),
\end{equation}
where 
\begin{equation}
        \phi_a(x) := \frac{1}{1-a}\mathrm{sign}(x) \max (0, |x|-a).
\end{equation}
For later purposes, we also write
\begin{equation}
    \widehat{E}(\sigma) = -\sum_{i_1,\ldots, i_m} \widehat{T}_{i_1\cdots i_m}\sigma_{i_1}\cdots \sigma_{i_m},  \label{eq:energy}
\end{equation}
where the sum over indices $i_1,\ldots, i_m$ allows repeated entries. In the classical DenseAM \eqref{eq:DAM}, we have $a=0$:
\begin{equation}
    \widehat{T}_{i_1\cdots i_m} = \phi_{a=0} (T_{i_1\cdots i_m})= \frac{1}{P}
\sum_{\mu=1}^P \xi^\mu_{i_1}\cdots \xi_{i_m}^\mu.  \label{eq:Ti1im}
\end{equation}
As we will see later, it is also interesting to consider a thresholded DenseAM (TDenseAM), where parameter $0\leq a<1$. The role of the parameter $a$ is to provide a ``noise floor" below which we reject couplings in $\widehat{E}$.  In other words, we wish to only keep terms in $\widehat{E}$ where we see a persistent signal across the observed patterns.

Once we have a learned energy function $\widehat{E}$, we simply perform gradient descent from any initial condition to deduce the output of the model, as is standard in the literature on EBMs.  In practice, we will perform gradient descent by choosing each bit (in a random order), and flipping it if doing so lowers the energy.  We update the value of $E$ before checking the next bit.   This algorithm terminates in a local minimum after a finite number of steps \cite{krotov2025modern}.

\section{LDPC codes}
Whether or not the DenseAM can efficiently learn will clearly depend on the choice $\mathcal{C}$ of patterns.  A natural idea is to choose $\mathcal{C}$ as the ground states of a frustration-free energy function $E$, which itself takes the form of \begin{equation}
    E(\sigma): = -\sum_{i_1,\ldots, i_m} J_{i_1\cdots i_m}\sigma_{i_1}\cdots \sigma_{i_m}.
\end{equation}
Then we would want to learn $\widehat{T}_{i_1\cdots i_m}=J_{i_1\cdots i_m}$.  There's a natural way to choose such $J_{i_1\cdots i_m}$.  If $E$ is frustration-free, we need, for every $\sigma \in \mathcal{C}$, $J_{i_1\cdots i_m}\sigma_{i_1}\cdots \sigma_{i_m}\ge 0$.   For simplicity, suppose each $J_{i_1\cdots i_m}\ge 0$.  Then we need $\sigma_{i_1}\cdots \sigma_{i_m}=1$ for any $J_{i_1\cdots i_m}>0$.  This is precisely the XORSAT problem in computer science, also known as a low-density parity-check (LDPC) code \cite{urbanke}.

It is more natural to switch to a binary representation: \begin{subequations} \label{eq:x-to-sigma}
    \begin{align}
        \xi_i^\mu &=: 1-2b_i^\mu, \\
        \sigma_i &=: 1-2x_i. 
    \end{align}
\end{subequations}
The ground states of $E$, or the elements of $\mathcal{C}$, are then the right null vectors of a parity-check matrix $H$: \begin{equation}
    \mathcal{C} = \lbrace b\in \mathbb{F}_2^n : Hb=0\rbrace ,
\end{equation}
where $\mathbb{F}_2 = \lbrace 0,1\rbrace$ and the parity-check matrix $H$ is $\mathbb{F}_2$-valued.   For every $J_{i_1\cdots i_m} \ne 0$, we can add a row $\alpha$ to the parity-check matrix with \begin{equation}
   H_{\alpha i} = \left\lbrace\begin{array}{ll} 1 &\ i\in \alpha =\lbrace i_1,\ldots, i_m\rbrace \\ 0 &\ i \notin \alpha \end{array}\right..
\end{equation}
By construction, notice that $\mathcal{C}$ has the structure of a vector space.  The dimension of this vector space is labeled $k$, and referred to as the number of logical bits.  Clearly then, $|\mathcal{C}|=2^k$.  In what follows, let \begin{equation}
    \widehat{\Delta_{\mathrm{B}}} := \left\lbrace \begin{array}{ll} 1 &\ \Delta_{\mathrm{B}} \text{ is even} \\  0 &\ \Delta_{\mathrm{B}} \text{ is odd} \end{array}\right..
\end{equation}

In practice, we will prefer to pick a random matrix $H$ and choose $J_{i_1\cdots i_m}$ accordingly.  We choose the random matrix $H$ as follows \cite{SipserSpielman1996}: pick a number $\Delta_{\mathrm{B}}<m$, and choose the rows of $H$ randomly such that each row has exactly $m$ 1s and each column has exactly $\Delta_{\mathrm{B}}$ 1s.  We now summarize some crucial properties of these random codes, leaving a more formal discussion to Appendix \ref{app:ldpc}.  A useful representation of an LDPC code is a Tanner graph (or a factor graph), depicted in Figure \ref{fig:tanner}.  The Tanner graph contains nodes that are either bits or checks, with an edge between bit $i$ and check $\alpha$ if and only if $H_{\alpha i}=1$.  LDPC codes then correspond to sparse Tanner graphs.  We generate LDPC codes randomly by choosing random\footnote{The precise distribution of such graphs that we sample (uniformly) from is detailed in Appendix \ref{app:ldpc}.} Tanner graphs where every bit node has degree $\Delta_{\mathrm{B}}$ and every check node has degree $\Delta_{\mathrm{C}}=m$.   Almost surely as $n\rightarrow \infty$, \begin{equation}
    k = \left(1 - \frac{\Delta_{\mathrm{B}}}{m} \right)n + \widehat{\Delta_{\mathrm{B}}}.
\end{equation}
Relatedly, the number $\mathcal{N}$ of parity-checks of weight $m$ is, with high probability, 
\begin{equation}
\mathcal{N} = \frac{\Delta_{\mathrm{B}}}{m}n + \mathrm{O}(1).  \label{eq:numberofchecks}    
\end{equation}

\begin{figure}
    \centering
    \includegraphics[width=0.5\linewidth]{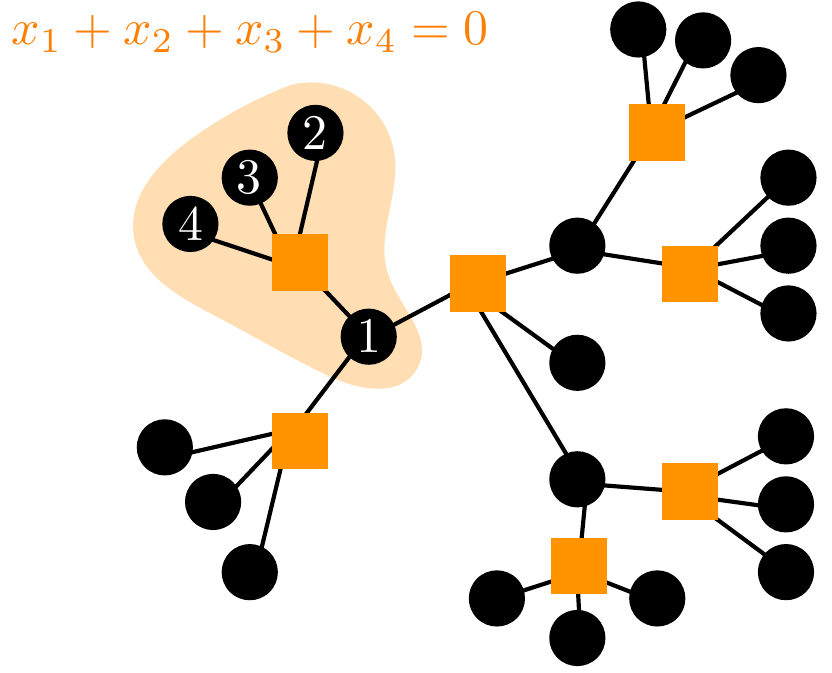}
    \caption{Illustration of a local patch of a Tanner graph of  a random code with $(\Delta_{\mathrm{B}},\Delta_{\mathrm{C}})=(3,4)$.  Bits $i$ are depicted as black circles; checks $\alpha$ are depicted as orange squares.  The number of small cycles in a Tanner graph is extremely small, so the graph will typically look locally tree-like.}
    \label{fig:tanner}
\end{figure}

There are simple, yet useful, observations about the typical weights $T_{i_1\cdots i_m}$.   If $\lbrace i_1,\ldots, i_m\rbrace$ is a real parity-check contained in $\mathrm{im}(H^{\mathsf{T}})$, we always have \begin{equation}
T_{i_1\cdots i_m} = 1. \label{eq:basic1}
\end{equation}
Otherwise, $\xi_{i_1}^\mu \cdots \xi^\mu_{i_m}$ equally likely to be $\pm1$; hence \begin{subequations}\begin{align}
    \mathbb{E}[T_{i_1\cdots i_m}] &= 0, \label{eq:zeromean} \\
\mathbb{E}\left[ {T_{i_1\cdots i_m}}^2\right] &= \frac{1}{P}. \label{eq:basic2}
\end{align}\end{subequations}
Hoeffding's inequality implies \begin{equation}
        \mathbb{P}\left(\left|T_I\right| > s\right) \le 2\mathrm{e}^{-Ps^2/2}. \label{eq:hoeffding}
    \end{equation} 
In fact, these observations explain why we should choose $\mathcal{C}$ as the codewords (ground states) of an LDPC code's $E$.   $\widehat{T}_{i_1\cdots i_m}$ are chosen by \eqref{eq:Ti1im} to be equal weight superpositions of the observed patterns.  So in general, TDenseAM does not perform Hamiltonian learning; it simply measures low-order correlations between spins when drawing samples uniformly at random from $\mathcal{C}$.  We are only guaranteed that these low-order correlations \emph{exactly} match the terms in $E$ itself, which follows from the equations above, when $E$ parametrizes a linear code.  For this reason, the (T)DAM architecture is perfectly suited to learn binary linear codes.  The suitability of the data set to the architecture will enable us to make rigorous statements about how the model learns.

\section{Learning random codes}
Now let us describe the learning process.  Here we give an intuitive and non-rigorous discussion; all scaling arguments are established rigorously in Appendix \ref{app:learning}.  

It is useful to first briefly review  how the regular DenseAM ($a=0$) architecture trains on random data \cite{krotov2016}.  Suppose that the number of patterns $P \sim n^\alpha$.   When $\alpha < m-1$,  \eqref{eq:DAM} is dominated by the actual learned patterns, in the sense that there are deep energy wells surrounding every seen pattern $\xi^\mu$.  (There are also ``mixed minima" corresponding to particular mixtures of patterns \cite{gayrard2026}.)  In contrast, once $\alpha>m-1$, the energy landscape becomes corrupted and overtrained. 

The story changes when we train on random codewords of an LDPC code.  Now the learned energy $\widehat{E}$ consists of two parts: the genuine parity-checks of the code, which contribute energy $\sim \mathcal{N}$ independent of the number of seen examples $P$, along with ``junk" pattern-sensitive terms, which contribute a typical energy $\sim \sqrt{n^m/P}$; these facts follow from \eqref{eq:basic1} and \eqref{eq:basic2}.  If $P\sim n^\alpha$ with $\alpha > m-2$, we see that the true parity-checks of the code begin to meaningfully shape the energy landscape: codewords are lower energy than random locations.    However, this does not mean that DenseAM will successfully recover unseen codewords.   Consider starting at an unseen codeword and perturbing the state by flipping any O(1) fraction of bits to align with a seen codeword.  The net energy change is \begin{equation}
    \mathrm{\Delta} \widehat{E} \sim \mathcal{N} - \frac{n^m}{P},
\end{equation}where the first positive contribution counts the number of violated parity-checks, while the second negative contribution comes from the huge energy gain from pattern alignment in \eqref{eq:DAM}. Using \eqref{eq:numberofchecks}, we see that until $\alpha \ge m-1$, the second term wins, so the DenseAM has a heavy bias towards reproducing seen patterns, just as when it is trained on random data.    In fact, we prove when $P/n^{m-1}$ is larger than a suitable O(1) constant, there are (with high probability) extensive energy barriers around each codeword.  This ensures that the denseAM will \emph{not} simply return a seen pattern -- it will at worst give an approximation of the true codeword (with a bounded fraction of errors) when initialized on an arbitrary codeword.

Next, let us turn to the opposite regime of $a\rightarrow 1$.   This extreme thresholding limit is equivalent to Valiant's (XOR)SAT learning algorithm \cite{valiant}, and to the ``reverse code engineering" problem \cite{tillich,sudan}.   For simplicity in the discussion that follows, let us assume that the first shown pattern is $\xi_i^1=1$ (we can always ``gauge fix" the training data so that this occurs).   After the first pattern, TDenseAM will presume that if $m$ is odd, every spin must be fixed to $\sigma_i=1$; if $m$ is even, $\sigma_i=-1$ is also allowed.   TDenseAM presumes that a parity-check is valid if and only if it has never seen it violated, and at this first step it has learned $\sim n^m$ checks.  All but $\mathcal{N}$ of them are spurious and should be unlearned during training.  As in \eqref{eq:zeromean}, each subsequent shown random codeword $\xi^\mu$ has a 50\% chance of violating each remaining spurious check.  The number of parity-checks in $\widehat{E}$ thus scales as $\mathcal{N} + n^m 2^{-P}$ and concentrates around this mean in the thermodynamic limit.  

Denote $\widehat{H}(P)$ as the learned parity-check matrix after $P$ patterns. $\widehat{H}(P)$  contains checks of weight $m$, $m-2$, etc.  Further denote\begin{equation}
    \widehat{k}(P) := \dim \left(\ker \widehat{H}(P)\right).
\end{equation}The number of minima of $\widehat{E}$ is $2^{\widehat{k}}$, because $\widehat{E}$ consists only of the parity-checks consistent with all seen codewords.  Since the minima of $\widehat{E}$ correspond to solutions of a linear equation, superpositions of observed codewords are also valid solutions.   

We conjecture that there is a sharp transition near \begin{equation}
    P_* = (m-1)\log_2 n.
\end{equation}
Indeed, we can show that once $P>P_*$, the number of spurious checks scales subextensively with $n$.  Even if all of these spurious checks are linearly independent (and they will be with high probability), they can at most reduce the rank to \begin{equation}
    \widehat{k} \approx k - \frac{n^m}{m!}2^{-P};
\end{equation} 
the first term $k=\mathrm{\Theta}(n)$ while once $P_* \sim (m-1+\epsilon)\log_2n$ for any $\epsilon>0$, the second term is subextensive.   

The behavior for $P<P_*$ depends on the value of $m$.  For $m=2$, we find that $\widehat{k}\approx 2^P$ for $P\ll \log_2k$.  This is because every LDPC code with weight-2 checks corresponds to a set of decoupled repetition codes: namely, there will be $k$ independent clusters of spins that must all agree within each cluster.  We need only observe codewords where any pair of two repetition code blocks take different values in order to distinguish them.  For example, we can see patterns $0000$, $0011$, $0101$ to learn that these 4 bits are all independent.

In contrast, for $m>2$, we instead expect that \begin{equation}
    \widehat{k}= P-1+\widehat{\Delta_{\mathrm{B}}} \label{eq:widehatkslow}
\end{equation}
until $P\sim (m-1)\log_2 n [1+\mathrm{o}(1)]$.    The reason for this drastic change in behavior is simple, yet profound: with higher weight checks, the $\sim n^m 2^{-P}$ surviving spurious checks necessarily have many linear dependencies (i.e. redundancies) until $P=(m-1)\log_2 n$.   These dependencies are also more complex compared to $m=2$: checks $w_1$ and $w_2$ can be violated by the next codeword while $w_1+w_2$ is a combination of non-violated checks.  We can crudely estimate that the typical number of ways to express $w_1+w_2$ as a linear combination of $L$ non-violated checks scales as $\gtrsim 2^{-PL}n^{mL/2-2m}$.  Hence when $P\lesssim cm \log_2n$ (with $c\le \frac{1}{2}$), removing $w_1$ and $w_2$ only increases $\widehat{k}$ by 1.  With a little more work, we can prove \eqref{eq:widehatkslow} as long as $P\lesssim cm\log_2 n$.  \eqref{eq:widehatkslow} is related to the fact that most random codes generated by picking $k\sim \log n$ random bitstrings as codewords are LDPC \cite{sudan}; it is then natural that a similar fact such as \eqref{eq:widehatkslow} can persist when given random codewords.  It may be a little surprising, though, that the structure of the code reveals itself so rapidly!

The dramatic change in learnability of the LDPC code is very easy to see when training a weight-$m$ TDenseAM on random bitstrings.  As $a\rightarrow 1$, the TDenseAM should eventually learn $\widehat{E}=0$, or $\widehat{k}=n$.   The random trajectories of $\widehat{k}(P)$ behave very differently when $m=2$ vs. $m>2$, exactly as predicted above.  A small numerical demonstration is shown in Figure \ref{fig:m2m4}.

\begin{figure}
    \includegraphics[width=0.49\textwidth]{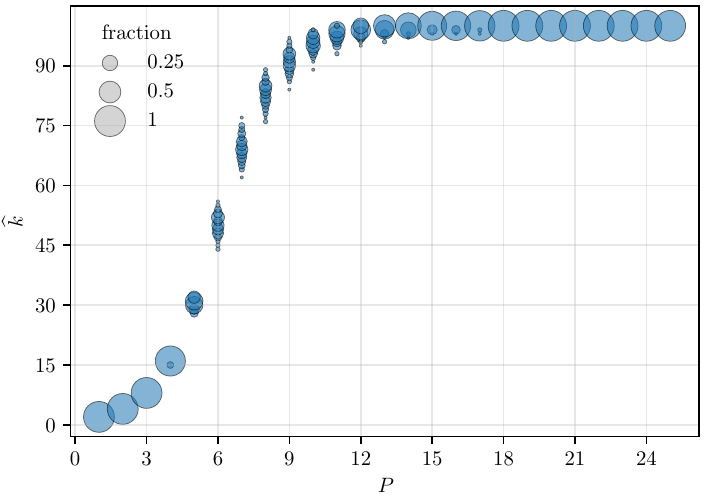}
    \includegraphics[width=0.49\textwidth]{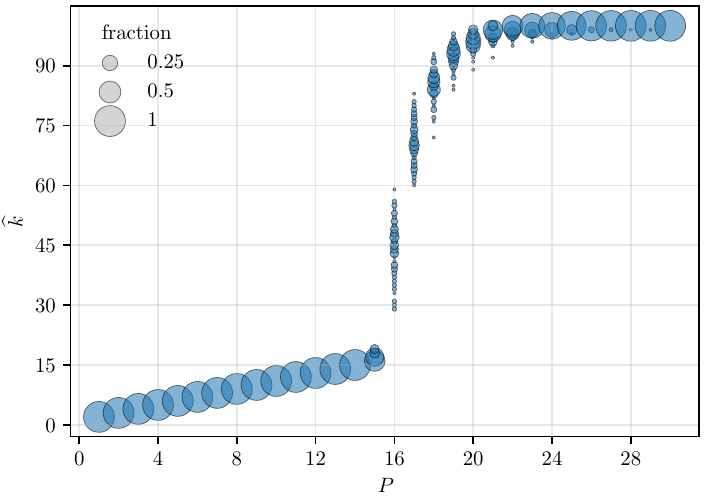}
    \caption{$\widehat{k}$ as a function of $P$ when training the $a\rightarrow 1$ TDenseAM on random bitstrings of length $n=100$.  The bubble plot depicts the probability of seeing a given value of $\widehat{k}$ at each $P$ when choosing a random sequence of training data.  Left: $m=2$.  Right: $m=4$, where \eqref{eq:widehatkslow} is clearly visible at small $P$.}
    \label{fig:m2m4}
\end{figure}

The punchline is that when $m>2$, TDenseAM learns ``slowly" at small $P$, and largely by memorization.  Indeed, for the $a\rightarrow 1$ learner, linear combinations of seen patterns are necessarily minima, so $\widehat{k} \rightarrow \widehat{k}+1$ learning is the least efficient possible, as long as every new seen pattern is linearly independent of the others (which almost surely happens as $n\rightarrow\infty$).   At a critical $P=P_*$, the spurious checks abruptly fall ``below" the noise threshold, and TDenseAM begins to rapidly \emph{generalize} beyond the training data it has seen.  

 This behavior is reminiscent of the grokking phenomenon seen in neural networks \cite{power2022grokking, gromov2023grokkingmodulararithmetic,MohamadiLiWuSutherland2024,morris2025much, li2026grokking}: one trains a neural network on ``algorithmic datasets" such as ours, and observes that first the neural network learns to reproduce training data it has seen, only ``discovering" the underlying rules generating the data long after reproducing training data. In neural networks, one could either talk about such a transition upon showing more training data, or optimizing network weights on a fixed dataset; in our context, only the former makes sense.   In the context of our XORSAT problem, the TDenseAM is basically guaranteed by construction to always correctly reproduce training data it has already seen.  The non-trivial step in the analogue ``grokking" is to generalize, which in this context amounts to correctly assigning unseen codewords (solutions) to local minima in $\widehat{E}$.    Figure \ref{fig:m2m4} demonstrates our analogue of this ``grokking" phenomenon in the $m>2$ TDenseAM when trained on random data.  In our context, the fact that grokking is a manifestation of the need to suppress spurious couplings in $\widehat{E}$ is very clear.  Firstly, it takes a long time for the spurious checks in $\widehat{E}$ to be suppressed relative to the true checks.  The abruptness of the transition is, in some sense, simply an artifact of the slow rate at which TDenseAM learns to remove spurious checks.  Secondly, once the spurious contribution to $\widehat{E}$ is comparable  to the true one, removing half of the spurious checks in each step leads to an explosive growth in generalizability, with nearly the entire code being discovered in a short training window.

When $0<a \ll 1$, we expect a rough ``interpolation" between the two extreme limits above. From \eqref{eq:hoeffding}, TDenseAM will almost surely learn the code so long as \begin{equation}
    P > \frac{2m}{a^2}\log_2 n. \label{eq:P1a2}
\end{equation}
This follows directly from a sub-optimal tail bound estimate on the probability of seeing a finite fraction $\ge \frac{1+a}{2}$ of codewords that violate (or not) any spurious check.  So, any amount of thresholding at all immediately leads to complete discovery of the underlying code -- and all of its $\exp(\mathrm{O}(n))$ codewords -- in $\mathrm{O}(\log n)$ patterns. 

We directly test this transition numerically by studying the behavior of the TDenseAM when it is initialized in an unseen codeword $\sigma^{\text{new}}$.  Suppose that after initializing $\sigma \rightarrow \sigma^{\text{new}}$ and running gradient descent using $\widehat{E}$, the final bit string is  $\sigma^{\mathrm{end}}$.   To test whether or not TDenseAM has simply memorized the patterns it has already seen, we probe \begin{equation}
    f_{\mathrm{new}} := \frac{1}{n} \min_\mu  \mathsf{d}_{\mathrm{H}}\left(\xi^\mu,  \sigma^{\text{end}}\right).
\end{equation}
Here $\mathsf{d}_{\mathrm{H}}$ denotes Hamming distance, or the number of bits in different states.  In the case where $m$ is even, we must also minimize over $-\xi^\mu$ as the TDenseAM will always have a global spin flip symmetry.  $f_{\text{new}}$ thus represents the fraction of bits in a ``new configuration" relative to any trained pattern:  $f_{\text{new}}$ should be 0 if we are memorizing, and O(1) if we are discovering new codewords.   Since we have fed the TDenseAM an actual codeword, we can test whether it has learned by asking whether or not the initial state $\sigma^{\text{new}}$ was in fact a local minimum, which is conveniently measured by the fraction of flipped spins: \begin{equation}
    f_{\text{flip}} := \frac{1}{n} \mathsf{d}_{\mathrm{H}}\left(\sigma^{\text{new}},  \sigma^{\text{end}}\right)
\end{equation}
Figure \ref{fig:mem-gen} shows the behavior of $f_{\text{flip}}$ and $f_{\text{new}}$ as a function of $P$.  When $a \ll 1$, we see a long period of \emph{strict} memorization -- just like in the $a=0$ DenseAM, we are overwhelmingly likely to simply reproduce a seen training example as the energy landscape is overwhelmingly shaped by these seen examples.  (Occasionally we do see the model settle at an unseen point, but this is rare.)  At a scale smaller than our bound in \eqref{eq:P1a2}, TDenseAM begins to efficiently generalize: $\widehat{E}$ develops local minima appear near typical unseen codewords, so it has learned the rules generating the data set.  The transition from memorization to generalization appears relatively sharp, but we have not analytically characterized this transition, and our simulations are clearly limited by finite-size effects.   The transition is clearly much more efficient when $a\rightarrow 1$.  In this case, $f_{\text{new}}$ is appreciable already at $P=2$, which is easily understood: TDenseAM in this limit also learns linear combinations of seen codewords.

\begin{figure}
    \includegraphics[width=0.49\textwidth]{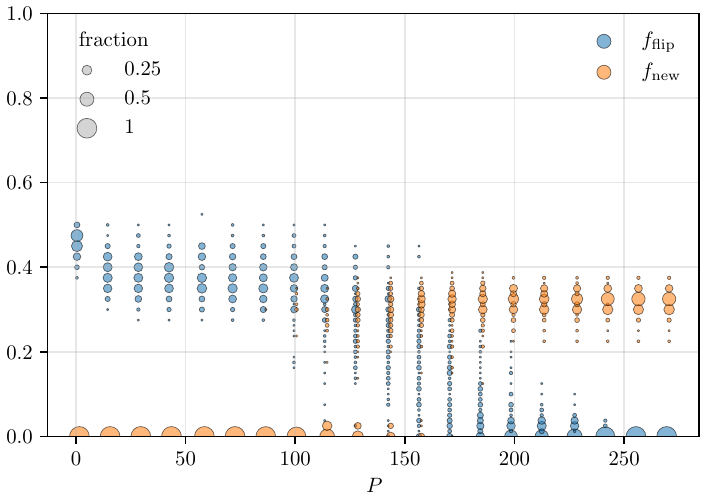}
    \includegraphics[width=0.49\textwidth]{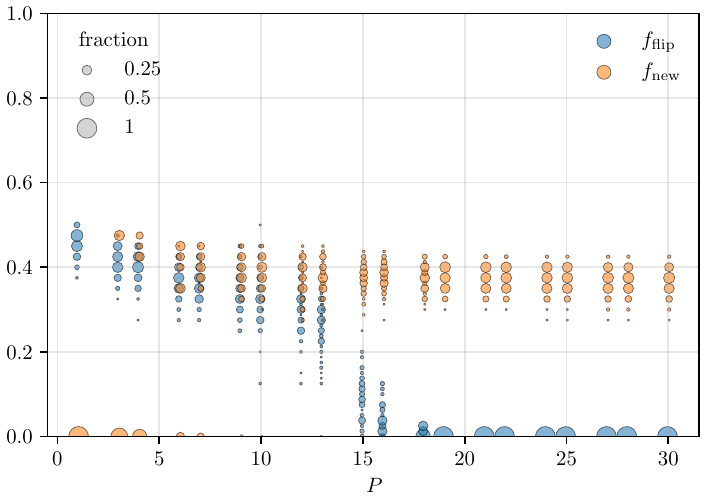}
    \caption{The memorization-generalization transition as a function of $P$, after initializing TDenseAM at unseen codewords.  The probability distributions of $f_{\text{flip}}$ and $f_{\text{new}}$ are depicted as bubble plots.   Training is done on a (3,4) code ($n=80$,  $k=20$).  Left: $a=0.2$.  Right: $a\rightarrow 1$.}
    \label{fig:mem-gen}
\end{figure}

It is plausible that right at the memorization-generalization transition, at any $a$, TDenseAM becomes very ``corrupted" by a glassy energy landscape in $\widehat{E}$.  This occurs in Figure \ref{fig:mem-gen} during the regime when both $f_{\text{flip}}$ and $f_{\text{new}}$ are sizable, right before generalization.  This is likely due to the glassy energy landscape of random LDPC codes themselves.  It would be interesting to investigate further in future work.

\section{Discovering latent variables}
It is often believed that when studying complex data patterns, some of the relevant degrees of freedom are not directly shown to the observer.  Such bits are called \emph{latent}.   A potential example is that visible bits correspond to the pixels in an image, while latent bits describe what the image represents.   There may not be a simple, local expression for $\widehat{E}$ without learning the latent variables.

Amazingly, we can also use TDenseAM to efficiently learn latent variables in XORSAT.   Let the first $n-\overline{n}$ bits be observed, and the remaining $\overline{n}$ bits be latent.  In ths section, we will denote observed bits with $x_i$ $(i=1,\ldots, n-\overline{n})$ and latent bits with $h_\alpha$ ($\alpha=n-\overline{n}+1,\ldots,n$).  We wish to generalize TDenseAM so that it can discover the latent bits.  Importantly, TDenseAM will not be told how many latent bits there are, nor where latent bits may be relative to the observed interaction graph -- it must discover all of this on its own.

A conservative way to discover a latent variable is to detect what we call an observable 3-shadow.  In the context of LDPC codes with weight-$m$ checks, this shadow will take the form of three $2m-2$-weight observable checks with a special structure: \begin{subequations}
    \begin{align}
        w_1 &= \overline{w}_2 + \overline{w}_3, \\ 
        w_2 &= \overline{w}_3 + \overline{w}_1, \\ 
        w_3 &= \overline{w}_1 + \overline{w}_2.
    \end{align}
\end{subequations}
 Here $\overline{w}_{1,2,3}$ are vectors of Hamming weight $m-1$.  Notice that $w_{1,2,3}$ are consistent with three smaller, weight $m$, checks that contain a latent variable $h$: \begin{equation}
            h + \overline{w}_1 =h + \overline{w}_2=h + \overline{w}_3 = 0. \label{eq:hwA0}
        \end{equation}
The key intuition is that if we observe three such patterns, rather than only two, we have clear insight as to where the latent variable must go.  If we only saw $w_1$, but not $w_{2,3}$, the decomposition into $\overline{w}_{2,3}$ is not knowable.  In principle, one could still introduce many candidate latent variables and keep only the ones which remain predictive.  To keep the story simple, however, we will restrict to learn latent variables when they have a 3-shadow as above.

We now summarize how to build a latent TDenseAM, leaving details to Appendix \ref{app:latent}.  The strategy is to run TDenseAM looking for weight-$(2m-2)$ checks.   We immediately accept checks of weight $\le m$, but also look for observable 3-shadows among the observed variables.   When we detect them, we introduce a latent variable to reduce them into acceptable checks with weight $\le m$.  The value of the latent variable should then be appended to all of the training data.  Then, we re-run TDenseAM with the learned latent variables treated as observed.  This process continues until we cannot discover any more latent variables.  An illustration is given in Figure \ref{fig:latent}. 

       \begin{figure}
    \centering
    \includegraphics[width=0.5\linewidth]{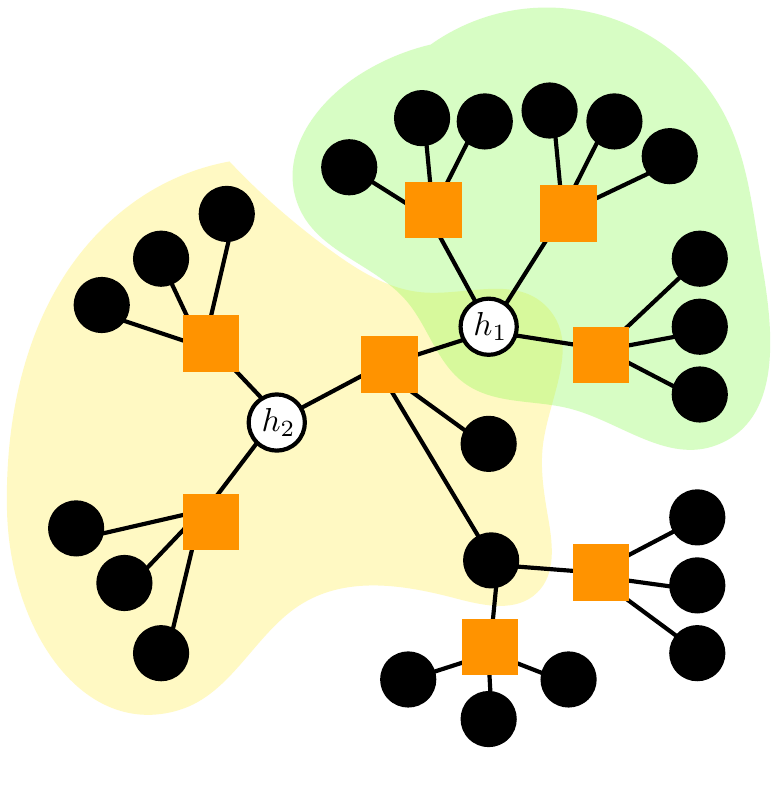}
    \caption{Illustration of a local patch of a Tanner graph of  a random code with observed (black) and latent (white) bits. The latent TDenseAM learner will detect $h_1$ since it has an observable 3-shadow.  Learning $h_1$ reveals $h_2$, which also subsequently has an observable 3-shadow.}
    \label{fig:latent}
\end{figure}

We prove in Appendix \ref{app:latent} that, almost surely as $n\rightarrow \infty$, latent TDenseAM can learn every latent variable correctly when training on random $(\Delta_{\mathrm{B}},\Delta_{\mathrm{C}})$ LDPC codes, if $5<\Delta_{\mathrm{B}}<\Delta_{\mathrm{C}}$, \begin{equation}
    \overline{n} < \gamma n \label{eq:overlinenbound}
\end{equation} for some O(1) constant $\gamma$, and the number of training patterns exceeds \begin{equation}
           P > \frac{4(m-1)}{a^2}\log_2 n.
       \end{equation}  
       In a nutshell, the proof uses the fact that locally, random codes look rather tree-like.  This property of the code, more accurately called expansion, guarantees that some latent variables will always have observable 3-shadows, and that any time an observable 3-shadow is found, it comes from an actual latent variable. This also guarantees that latent TDenseAM will never introduce incorrect constraints.

       The guarantee above does not rely on the latent bits being distributed randomly amongst the code; every subset of $\overline{n}$ latent bits will be learned, as long as we have \eqref{eq:overlinenbound}.   This is interesting because it means that we can even learn latent bits which are only observable after $\sim \log n$ rounds of running TDenseAM.  The correlations used to detect these deep latent variables will originally be extensive correlations among finite fractions of all of the observed bits.  The expansion of random codes, along with the certification of the shallower latent bits, guarantees that the deepest latent variable is genuine and detectable.


\section{Outlook}
We have presented a simple and largely solvable model for how a (thresholded) dense associative memory learns on a specific algorithmic dataset -- randomly chosen codewords of an expander LDPC code.  This architecture first learns by memorization of observed patterns $\xi^\mu$, and later abruptly begins to generalize well beyond the training data, in a sharp transition somewhat reminiscent of grokking.   The same architecture is also capable of discovering latent variables efficiently, providing a simple and solvable cartoon for the role of deep latent variables as sparsifications of the rules that generated the dataset.

Although we have not explicitly verified it, it is almost certain that the models described here can easily train on corrupted LDPC codewords, and any thresholding $a>0$ ensures learnability in $\mathrm{\Theta}(\log n)$ corrupted samples (so long as the corruption is reasonably independent);  it would be interesting to study the dynamics of learning more carefully in this setting.  More non-trivially, carefully chosen EBMs can provably efficiently learn other problems: for example, certain modular addition tables are discovered \cite{lucasfuture} doubly-exponentially faster in EBMs than in one-hot neural networks \cite{power2022grokking, gromov2023grokkingmodulararithmetic,MohamadiLiWuSutherland2024,morris2025much, li2026grokking}.  Perhaps the ideas we have fleshed out in a simplified setting may prove relevant more broadly in a theory of how AI learns structured data.

\section*{Acknowledgments}
We thank MPI-PKS for hospitality as this work was initiated.   ChatGPT 5.5/6 Pro was used to help write codes and to assist in proving the main theorems.  Codex 5.5 Medium wrote the code to make figures on small sample sizes for the main text.

\begin{appendix}

\makeatletter


\let\@sectioncntformat\@seccntformat

\let\@hangfrom@section\@hang@from


\renewcommand\section{\@ifstar{\AndySectionStar}{\@startsection{section}{1}{\z@}%

  {-3.5ex \@plus -1ex \@minus -.2ex}%

  {2.3ex \@plus .2ex}%

  {\normalfont\Large\bfseries}}}

\renewcommand\subsection{\@ifstar{\AndySubsectionStar}{\@startsection{subsection}{2}{\z@}%

  {-3.25ex\@plus -1ex \@minus -.2ex}%

  {1.5ex \@plus .2ex}%

  {\normalfont\large\bfseries}}}

\renewcommand\subsubsection{\@ifstar{\AndySubsubsectionStar}{\@startsection{subsubsection}{3}{\z@}%

  {-3.25ex\@plus -1ex \@minus -.2ex}%

  {1.5ex \@plus .2ex}%

  {\normalfont\normalsize\bfseries}}}

\makeatother

\renewcommand{\thesubsection}{\thesection.\arabic{subsection}}
\renewcommand{\theequation}{\thesection.\arabic{equation}}
\renewcommand{\thesubsubsection}{\thesubsection.\arabic{subsubsection}}
\makeatletter
\renewcommand\section{\@ifstar{\AndySectionStar}{\@startsection {section}{1}{-\parindent}%
  {-3.5ex \@plus -1ex \@minus -.2ex}%
  {2.3ex \@plus .2ex}%
  {\normalfont\Large\bfseries}}}
\renewcommand\subsection{\@ifstar{\AndySubsectionStar}{\@startsection{subsection}{2}{-\parindent}%
  {-3.25ex\@plus -1ex \@minus -.2ex}%
  {1.5ex \@plus .2ex}%
  {\normalfont\large\bfseries}}}
\renewcommand\subsubsection{\@ifstar{\AndySubsubsectionStar}{\@startsection{subsubsection}{3}{-\parindent}%
  {-3.25ex\@plus -1ex \@minus -.2ex}%
  {1.5ex \@plus .2ex}%
  {\normalfont\normalsize\bfseries}}}
\makeatother

\section{Review of binary linear codes}\label{app:ldpc}

\subsection{Basic definitions}
In the text that follows we will use some convenient shorthand:  $[n] = \lbrace 1,2,\ldots, n\rbrace$ refers to the set of $n$ elements labeled by the first $n$ integers.  Similarly we will just write $[I]$ to denote the set of all objects that could be labeled with $I$: in particular, if $I$ denotes an $m$-tuple of the bits which can be a coupling of the dense associative memory, then $[I] = \lbrace I \subset [n] : |I| = m\rbrace$, etc. We denote the indicator function with \begin{equation}
    \mathbb{I}[X] := \left\lbrace\begin{array}{ll}1 &\ \text{$X$ is true} \\ 0&\ \text{$X$ is false} \end{array}\right..
\end{equation}
Lastly, we will sometimes refer to a vector $x\in \mathbb{F}_2^n$ interchangably with $I \subset [n] $ defined by $I:=\lbrace i\in [n]:x_i=1\rbrace$.  So the statement $I\in \mathbb{F}_2^n$, and similar statements, should be understood in this light.  

Consider a vector space $V$ over the two-element field $\mathbb{F}_2 = \lbrace 0,1\rbrace$, defined by addition and multiplication mod 2.  The mod 2 will always be implicit when it is clear from context that operations are done over $\mathbb{F}_2$. A useful representation of $V \cong \mathbb{F}_2^k$ will be given by the following proposition, which summarizes basic facts about linear algebra over finite fields: \begin{prop}
    Define a matrix $G \in \mathbb{F}_2^{k\times n}$ such that \begin{equation}
        V = \mathrm{im}(G^{\mathsf{T}}),
    \end{equation}
    namely the rows of $G$ are linearly independent generators of $V\subset \mathbb{F}_2^n$.  Then there exists a (non-unique) matrix $H\in \mathbb{F}_2^{m\times n}$ with $m\ge n-k$ such that $HG^{\mathsf{T}} = 0$, and 
    \begin{equation}
    \ker(H) = \mathrm{im}(G^{\mathsf{T}}).    
    \end{equation}
   We will denote $V=\ker(H)$ and $V^\perp := \mathrm{im}(H^{\mathsf{T}})$.\footnote{Note that on $\mathbb{F}_2$, $V\cap V^\perp$ can be non-empty.}
\end{prop}


\begin{defn}[Canonical form]
    By Gaussian elimination we can always find a canonical form for the generator matrix $G$ and parity-check matrix $H$: \begin{subequations}\begin{align}
        G &= \left(\begin{array}{cc} I_{k\times k}  &\ A \end{array} \right), \\
         H &= \left(\begin{array}{cc} A^{\mathsf{T}}  &\ I_{(n-k)\times (n-k)} \end{array} \right), 
    \end{align}\end{subequations}
    Here $A$ is a $k\times (n-k)$ matrix.   Notice that a canonical form of $H$ need not be sparse, even if there exists a sparse representation of $H$.
\end{defn}

\subsection{Useful properties}

 The following observation is crucial to our subsequent analysis:
\begin{prop}\label{prop:xiI}
    Let $a_i\in\mathbb{F}_2^n$.  Define $I\subseteq \lbrace 1,\ldots, n\rbrace := \lbrace i : a_i=1\rbrace$ and \begin{equation}
        \xi_I^\mu := \prod_{i\in I} \xi_i^\mu .
    \end{equation}  Then  \begin{equation}
        \frac{1}{2^{\dim V}}\sum_{\xi \in V}\xi_I = \left\lbrace \begin{array}{ll} 1 &\ a\in V^\perp\\ 0 &\ a\notin V^\perp\end{array}\right.. \label{eq:asum}
    \end{equation}
    Here $\xi \in V$ is understood via the mapping \eqref{eq:x-to-sigma}.   
\end{prop}
\begin{proof}
   Observe that \begin{equation}
    \xi_I = (-1)^{\sum_{i=1}^n a_i b_i}.
    \end{equation}
    If $b\in \mathrm{im}(H^T) = V^\perp$, then $b\cdot a = 0$ for every $b \in \ker(H) = V$,  establishing the first line of our identity.   
    
    As for the second line, if $a\notin \mathrm{im}(H^{\mathsf{T}}) = \ker(G)$ then by definition there exists $c\in V = \mathrm{im}(G^{\mathsf{T}})$ such that $c\cdot b =1$.  For each $b$ in the sum \eqref{eq:asum}, it is paired with a unique partner $b+c$ in the sum (since $c+(c+b)=b$), and $(-1)^{a \cdot b} + (-1)^{a\cdot (b+c)} = 0$.  Summing over all such pairs of partners establishes the second line of \eqref{eq:asum}.
\end{proof}

\begin{prop}\label{prop:linindcheck}
    Let $\lbrace w_a\rbrace_{a=1}^q$ be linearly independent non-trivial cosets in $\mathbb{F}_2^n/\im(H^{\mathsf{T}})$.  Then there exist linearly independent codewords $\lbrace v_a\rbrace_{a=1}^q  \subset \ker(H)$ such that \begin{equation}
        w_a \cdot v_{a^\prime} = \delta_{aa^\prime}.
    \end{equation}
\end{prop}
\begin{proof}
    First look for a $v_1$ such that $w_1\cdot \tilde v_1=1$.  If such a $\tilde v_1$ did not exist, then $w_1\in \ker(H)$ contradicting our assumptions.   Similarly, we construct candidate $\tilde v_a$.  If it was not possible to choose $q$ linearly independent $\tilde v_a$, then there would exist some combination of $w_a$ which was orthogonal to every possible codeword, in violation of the assumption that the $w_a$ were independent.   Once we have found $q$ linearly independent $\tilde v_a$, we can do Gaussian elimination to find an invertible matrix $M_{ab}$ such that $v_a = M_{ab}\tilde v_b$.
\end{proof}

\subsection{Random expander codes}
In this subsection, we will review established facts about random expander codes \cite{urbanke}, emphasizing points which are relevant for our discussion. Its purpose is just to summarize the key properties of random codes which provide provable guarantees on efficient learning of TDenseAM for the remainder of the paper.  The reader should feel free to only refer to this section as results are referenced below.

\begin{defn}[$(\Delta_{\mathrm{B}},\Delta_{\mathrm{C}})$-LDPC codes]
    A $(\Delta_{\mathrm{B}},\Delta_{\mathrm{C}})$-LDPC code corresponds to parity-check matrix $H$ for which each row has Hamming weight $\Delta_{\mathrm{C}}$ and each column has Hamming weight $\Delta_{\mathrm{B}}$.\footnote{Often, we will choose $m=\Delta_{\mathrm{C}}$, because $\Delta_{\mathrm{C}}$ will be the natural order of interactions to learn in the dense associative memory!}  In other words, each bit connects to $\Delta_{\mathrm{B}}$ checks, and each check connects to $\Delta_{\mathrm{C}}$ bits.  It is often useful to define a set of bits $B$ and checks $C$; in the previous discussion we had informally used $B=\lbrace 1,\ldots, n\rbrace$.  Given a subset of bits $B_0\subseteq B$, we define its boundary \begin{equation}
        \lbrace c\in C : H_{cb} = 1\text{ for some }b \in B_0\rbrace =:  \partial B_0 \subseteq C.
    \end{equation}
    An analogous definition (with $H\rightarrow H^{\mathsf{T}}$) holds for the boundary of a check set.
\end{defn}

\noindent It is actually quite easy to generate these codes randomly due to a clever algorithm known as the configuration model \cite{bollobas}: one simply generates random candidate codes and rejects them if they are ``bad" -- this will succeed with high probability after only a finite number of attempts, even as $n\rightarrow \infty$.  We summarize the result without proof below: \begin{thm}\label{thm:configurationmodel}
    Propose a $(\Delta_{\mathrm{B}},\Delta_{\mathrm{C}})$-LDPC code by randomly connecting each bit to a check $\Delta_{\mathrm{B}}$ times, and stopping if the check has more than $\Delta_{\mathrm{C}}$ edges.   The probability that there is at least one bit/check pair that are connected more than once is, as $n\rightarrow \infty$, at most $1-\mathrm{e}^{-(\Delta_{\mathrm{B}}-1)(\Delta_{\mathrm{C}}-1)/2}$ \cite{bekessy}. 
\end{thm}

\noindent We will use this simple generation algorithm below to prove some very useful facts about random LDPC codes.  A key property about random codes is that they are expander codes, as defined below: 

\begin{defn}[Left/right-expansion]\label{def:expansion}
    A code has $(\gamma,\delta)$-left expansion if, for every set of bits $B_0\subset B$ with $|B_0| < \gamma n$, $|\partial B_0| \ge \Delta_{\mathrm{B}}(1-\delta)|B_0|$.  It has $(\gamma,\delta)$-right expansion if the analogous property holds for checks: $|\partial C_0| \ge \Delta_{\mathrm{C}}(1-\delta)|C_0|$.  If both hold, the code has $(\gamma,\delta)$-expansion.
\end{defn}

\begin{prop}[Unique-neighbor expansion]\label{prop:expansion}
    Given a code $(B,C)$ with $(\gamma,\delta)$-expansion, consider any $B_0 \subset B$ with $|B_0| < \gamma n$.   Let $\partial_{\mathrm{u}}B_0\subset C$ be the number of unique-neighbor checks which connect only to a single bit in $B_0$: \begin{equation}
        \partial_{\mathrm{u}}B_0 := \lbrace c\in C : |\partial c\cap B_0|=1\rbrace .
    \end{equation}
    Then \begin{equation}
        |\partial_{\mathrm{u}}B_0| \ge (1-2\delta)\Delta_{\mathrm{B}}|B_0|. \label{eq:partial_u_bound}
    \end{equation}
\end{prop}
\begin{proof}
    The number of edges between bits in $B_0$ and checks in $C$ is given by $\Delta_{\mathrm{B}}|B_0|$.  By Definition \ref{def:expansion}, there are edges to $\Delta_{\mathrm{B}}(1-\delta)|B_0|$ checks.  Comparing the total number of edges to the total number of edges to unique checks, vs. edges to non-unique checks, we see that \begin{equation}
        \Delta_{\mathrm{B}}|B_0| \ge |\partial_{\mathrm{u}}B_0| + 2\left( \Delta_{\mathrm{B}}(1-\delta)|B_0| - |\partial_{\mathrm{u}}B_0| \right),
    \end{equation}
    which implies \eqref{eq:partial_u_bound}.
\end{proof}

\begin{cor}\label{cor:linearconfinement}
    For any $x$ with $|x| \le \gamma n$, a code with $(\gamma,\delta)$ expansion has \begin{equation}
        |Hx| \ge \Delta_{\mathrm{B}}(1-2\delta)|x|. \label{eq:linearconfinement}
    \end{equation}
    The code distance $d \ge \gamma n$.\footnote{This is not an optimal bound, but it suffices for our conceptual purposes in this paper -- namely, establishing that the code distance can be $\mathrm{O}(n)$ whenever we can establish expansion.}
\end{cor}
\begin{proof}
    These are immediate consequences of Proposition \ref{prop:expansion}.  Any check in $\partial_{\mathrm{u}}B_0$ must be violated, so \eqref{eq:linearconfinement} follows from \eqref{eq:partial_u_bound}.  Since $|Hx|=0$ on any codeword, if $x\in \ker(H)$ either $x=0$ or $|x|>\gamma n$.
\end{proof}

\begin{cor}
        In the setting of Proposition \ref{prop:expansion}, we are guaranteed that there exists $b\in B_0$ with $|\partial b \cap \partial_{\mathrm{u}}B_0| \ge m$ whenever \begin{equation}
        \delta < \frac{1+\Delta_{\mathrm{B}}-m}{2\Delta_{\mathrm{B}}}. \label{eq:delta_m_bound}
    \end{equation}
\end{cor}
\begin{proof}
    From \eqref{eq:partial_u_bound}, the average number of unique-neighbor checks in $\partial B_0$ is $\Delta_{\mathrm{B}}(1-2\delta)$.  If $\Delta_{\mathrm{B}}(1-2\delta) > m-1$, there must be at least one bit with $|\partial b \cap \partial_{\mathrm{u}}B_0| \ge m$.  Hence we find \eqref{eq:delta_m_bound}.
\end{proof}

The existence of useful codes is established by the following theorem, which we state without proof, although we do note that the bound on $k$ comes directly from the rank-nullity theorem on $\mathbb{F}_2$, and the techniques are similar to those discussed at the end of this subsection.

\begin{thm}\label{thm:randomcodes}
    Take $\Delta_{\mathrm{B}}<\Delta_{\mathrm{C}}$.  Choose a $(\Delta_{\mathrm{B}},\Delta_{\mathrm{C}})$-LDPC code uniformly at random via the configuration model (Theorem \ref{thm:configurationmodel}).  Then the code has a linear number of logical bits: \begin{equation}
        k = n \left(1-\frac{\Delta_{\mathrm{B}}}{\Delta_{\mathrm{C}}}\right) + \left\lbrace\begin{array}{ll} 1 &\ \Delta_{\mathrm{B}} \text{ even} \\ 0 &\ \Delta_{\mathrm{B}} \text{ odd} \end{array}\right.,
    \end{equation}and  with probability $1-\mathrm{O}(n^{1-\delta \Delta_{\mathrm{B}}})$, the code is a $(\gamma,\delta)$-expander for any $\delta > 1/\Delta_{\mathrm{B}}$ and any $0<\gamma<\gamma_{\mathrm{max}}$, where $\gamma_{\mathrm{max}}$ solves \begin{equation}
       \left(1-\frac{1}{\Delta_{\mathrm{B}}}\right)h(\gamma_{\mathrm{max}}) = \frac{1}{\Delta_{\mathrm{C}}}h(\Delta_{\mathrm{C}}(1-\delta)\gamma_{\mathrm{max}}) + \Delta_{\mathrm{C}}\gamma_{\mathrm{max}}   h\left(\frac{1}{\Delta_{\mathrm{C}}(1-\delta)}\right)
    \end{equation}with \begin{equation}
        h(x) := -x\log x - (1-x)\log(1-x). \label{eq:binaryentropy}
    \end{equation}
\end{thm}

\noindent Lastly, the following result tells us that the interaction graph of a random LDPC code does not contain very many interesting small subgraphs, with very high probability.  In other words, random LDPC codes look locally tree-like.

\begin{prop}
    Let $G$ be a bipartite graph with $V_0$ vertices and $E_0$ edges, with $V_0$ and $E_0$ $\mathrm{O}(1)$ constants.   
    Pick a $(\Delta_{\mathrm{B}},\Delta_{\mathrm{C}})$ LDPC code uniformly at random, and let $N_G$ be the number of subgraphs isomorphic to $G$. As $n\rightarrow \infty$, there exists a constant $C<\infty$ such that \begin{equation}
        \mathbb{E}[N_G] \le C n^{V_0-E_0}. \label{eq:ENG}
    \end{equation}
    In words, there are a finite number of short cycles in an LDPC code, and these cycles are all far from each other. \label{prop:subgraph}
\end{prop}
\begin{proof}
    We can use Theorem \ref{thm:configurationmodel} to constrain the properties of a random LDPC code:  suppose that a random pairing of bits and checks has property $X$ with probability $\mathbb{P}_{\text{bare}}(X)$.  Then, letting $q=\mathrm{e}^{-(\Delta_{\mathrm{B}}-1)(\Delta_{\mathrm{C}}-1)/2}$, the probability the event happens on an actual LDPC code must obey \begin{equation}
        \mathbb{P}_{\mathrm{LDPC}}(X) q + \mathbb{P}_{\text{not LDPC}}(X)(1-q) = \mathbb{P}_{\text{bare}}(X),
    \end{equation}
    meaning that \begin{equation}
        \mathbb{P}_{\mathrm{LDPC}}(X) \ge \frac{\mathbb{P}_{\text{bare}}(X)}{q}. \label{eq:PbarePLDPC}
    \end{equation}
    Now, in the configuration model, it is very easy to bound the probability of seeing a subgraph $G$.  If we demanded that the subgraph $\lbrace 1,\ldots, m_0\rbrace \subset B$ and $\lbrace m_0+1,\ldots, V_0\rbrace \subset C_0$ was isomorphic to $G$, this would occur with probability \begin{equation}
        \mathbb{P}\left(G \sim \lbrace 1,\ldots, m_0\rbrace  \cup \lbrace m_0+1,\ldots, V_0\rbrace \right) = \prod_{j=0}^{E_0-1} \frac{1}{\Delta_{\mathrm{B}}n - j} < (\Delta_{\mathrm{B}}n-E_0)^{-E_0}.
    \end{equation}
    Now enumerating over the total number of possible such subgraphs, we bound \begin{equation}
        \mathbb{E}[N_G] < \frac{n^{m_0}}{(\Delta_{\mathrm{B}}n-E_0)^{E_0}}\left(\frac{\Delta_{\mathrm{B}}}{\Delta_{\mathrm{C}}}n\right)^{V_0-m_0} < qC n^{V_0-E_0} \label{eq:Pbaresubgraph}
    \end{equation}
    for some $C<\infty$.  Combining \eqref{eq:PbarePLDPC} and \eqref{eq:Pbaresubgraph} we obtain \eqref{eq:ENG}.
\end{proof}

The remaining results in this section are a collection of scattered but useful facts about expander codes, which will be useful in establishing how TDenseAM trains on random codes.

\begin{thm}
    Pick a $(\Delta_{\mathrm{B}},\Delta_{\mathrm{C}})$ LDPC code uniformly at random with $2<\Delta_{\mathrm{B}}<\Delta_{\mathrm{C}}$.   If there are $m$ checks, take $s\in \mathbb{F}_2^m$ and define \begin{equation}
        |s|_{\mathrm{red}} := \left\lbrace \begin{array}{ll}  |s| &\ \Delta_{\mathrm{B}} \text{ odd} \\ \min(|s|,m-|s|) &\ \Delta_{\mathrm{B}} \text{ even} \end{array} \right. .
    \end{equation} Then almost surely as $n\rightarrow \infty$, every vector $s$ obeys \begin{equation}
        |H^{\mathsf{T}}s| \ge \kappa |s|_{\mathrm{red}} \label{eq:Cprimeinequality}
    \end{equation}
    for some $\mathrm{O}(1)$ constant $\kappa>0$.  This property is known as \textbf{check soundness}. \label{thm:checksoundness}
\end{thm}

\begin{proof}
Theorem \ref{thm:randomcodes} shows that the dimension of the codespace is minimal -- up to a sole redundant check when $\Delta_{\mathrm{B}}$ is even, all of the checks are independent. This redundancy is what requires us to look at the reduced $|s|_{\mathrm{red}}$ weight in the even case. 

    The desired property follows immediately from Definition \ref{def:expansion} if $|s|< \gamma n$, so we only need to consider the case where $|s|>\gamma n$. Here, we follow the same strategy as the proof of Proposition \ref{prop:subgraph}.  Since $|s|>\gamma n$, each bit sees one of its edges connected to a check in $s$ with probability no smaller than $p:=\Delta_{\mathrm{B}}\gamma $.   Our goal is to bound the bare probability that exactly $K := \kappa |s|_{\mathrm{red}}$ bits have an \emph{odd} number of edges connecting to $s$.   
    
    We can use the method of generating functions to bound this.  Define the odd and even parts of the binomial distribution: \begin{subequations}\begin{align}
        B_{\mathrm{even}}(z) &:= \sum_{j \text{ even}} \left(\begin{array}{c} \Delta_{\mathrm{B}} \\ j \end{array}\right) z^j, \\
        B_{\mathrm{odd}}(z) &:= \sum_{j \text{ odd}} \left(\begin{array}{c} \Delta_{\mathrm{B}} \\ j \end{array}\right) z^j.
   \end{align} \end{subequations}
    Then define 
    \begin{equation}
        G(z):=\sum_{n=0}^\infty G_n z^n := B_{\text{even}}(z)^{n-K}B_{\text{odd}}(z)^K.
    \end{equation}
    In the random generating model of Theorem \ref{thm:configurationmodel}, we see that \begin{align}
        \mathbb{P}_{\text{bare}}(K\text{ odd degree bits}) &= \left(\begin{array}{c} n \\ K \end{array}\right)\left(\begin{array}{c} \Delta_{\mathrm{B}}n/\Delta_{\mathrm{C}} \\ |s| \end{array}\right)\left(\begin{array}{c} \Delta_{\mathrm{B}}n \\ \Delta_{\mathrm{C}}|s| \end{array}\right)^{-1}G_{\Delta_{\mathrm{C}}|s|} \notag \\
        &\le \left(\begin{array}{c} n \\ K \end{array}\right)\left(\begin{array}{c} \Delta_{\mathrm{B}}n/\Delta_{\mathrm{C}} \\ |s| \end{array}\right)\left(\begin{array}{c} \Delta_{\mathrm{B}}n \\ \Delta_{\mathrm{C}}|s| \end{array}\right)^{-1} \inf_z \left(z^{-\Delta_{\mathrm{C}}|s|}G(z) \right).
    \end{align}
     Let $|s| = \eta n$; note that $\gamma <\eta \le \Delta_{\mathrm{B}}/\Delta_{\mathrm{C}}$.  Notice that as $n\rightarrow \infty$, \begin{equation}
      \frac{1}{n} \log \inf_z \left(z^{-\Delta_{\mathrm{C}}|s|}G(z) \right) = \inf_z \left( \log B_{\text{even}}(z) - \Delta_{\mathrm{C}} \eta \log z\right).
    \end{equation}
   Observe that $B_{\text{even}}(z) \le (1+z^{\Delta_{\mathrm{B}}/(\Delta_{\mathrm{B}}-1)})^{\Delta_{\mathrm{B}}-1}$.  Standard manipulations then reveal that \begin{align}
       \inf_z \left( \log B_{\text{even}}(z) - \Delta_{\mathrm{C}} \eta \log z\right) &= (\Delta_{\mathrm{B}}-1) \inf_u \left( \log(1+u) - \frac{\Delta_{\mathrm{C}}\eta}{\Delta_{\mathrm{B}}}\log u\right) 
       \notag \\
       &=  (\Delta_{\mathrm{B}}-1)h\left(\frac{\Delta_{\mathrm{C}}\eta}{\Delta_{\mathrm{B}}}\right),
   \end{align}
 recalling the definiton of $h(x)$ in \eqref{eq:binaryentropy}.   We see that as $n\rightarrow \infty$, \begin{equation}
        \mathbb{P}_{\text{bare}}(K\text{ odd degree bits})  \le   \mathrm{e}^{n [h(\kappa) -\psi ]} \label{eq:2psi}
    \end{equation}
    where \begin{equation}
        \psi = \left[\Delta_{\mathrm{B}} \left(1-\frac{1}{\Delta_{\mathrm{C}}}\right) - (\Delta_{\mathrm{B}}-1) \right] h\left(\frac{\Delta_{\mathrm{C}}\eta}{\Delta_{\mathrm{B}}}\right) = \left(1-\frac{\Delta_{\mathrm{B}}}{\Delta_{\mathrm{C}}}\right)h\left(\frac{\Delta_{\mathrm{C}}\eta}{\Delta_{\mathrm{B}}}\right) >0.
    \end{equation}
    Choose any $\kappa$ obeying $h(\kappa) < \psi$.  Combining \eqref{eq:PbarePLDPC} and \eqref{eq:2psi}, the probability of violating \eqref{eq:Cprimeinequality} vanishes exponentially as $n\rightarrow \infty$.
\end{proof}

\begin{cor}\label{cor:usefulfora=0}
        In a random $(\Delta_{\mathrm{B}},\Delta_{\mathrm{C}})$ LDPC code with $2<\Delta_{\mathrm{B}}<\Delta_{\mathrm{C}}$,  almost surely as $n\rightarrow \infty$, there are no non-zero elements of $\mathrm{im}(H^{\mathsf{T}})$ with Hamming weight smaller than $\Delta_{\mathrm{C}}$.   Moreover, the number of distinct weight $\Delta_{\mathrm{C}}$ checks is almost surely given by \begin{equation} \label{eq:Nm}
            \mathcal{N}_m = (1+\mathrm{o}(1))\frac{\Delta_{\mathrm{B}}}{\Delta_{\mathrm{C}}} n .
        \end{equation} 
    \end{cor}
    \begin{proof}
        By Theorem \ref{thm:checksoundness}, it suffices to look at small combinations of checks with $|s| < K^\prime$, where we can then apply the expansion guarantee of Theorem \ref{thm:randomcodes}.  For any $\delta > 1/\Delta_{\mathrm{B}}$, we know that $|H^{\mathsf{T}}s| \ge \Delta_{\mathrm{C}}(1-2\delta)|s|$.  Since $\Delta_{\mathrm{B}} \ge 3$, we can choose $\delta $ arbitarily close to $1/3$.   So if $|s|\ge 3$, the claim is established.  
        
       Thus let us discuss the case of $|s|=2$, as $|s|=1$ follows analogously.  Here we can apply Proposition \ref{prop:subgraph}.  For any $\Delta_{\mathrm{C}}\ge 4$, to find a check of weight smaller than $\Delta_{\mathrm{C}}$, we would need to find a subgraph of the Tanner graph (Figure \ref{fig:tanner}) isomorphic to two bits connected to $\ge 3$ checks, one check connected to the same bit $\ge 3$ times, etc.: any graph that contains $\ge 2$ independent cycles, namely has $V_0<E_0$; almost surely such a graph never appears.  There are a finite number of such graphs to verify below a fixed number of checks (supplied by Theorem \ref{thm:checksoundness}), so the probability any one of them fails vanishes as $n\rightarrow \infty$ by a union bound.

        Finally, let us verify that the number of checks of weight $\Delta_{\mathrm{C}}$ is $\mathrm{\Theta}(n)$.  We have at least $\Delta_{\mathrm{B}}n/\Delta_{\mathrm{C}}$ checks by the construction of the code.  Similar to the previous paragraph, we need at least a cycle of length 4 in the Tanner graph to find any redundant checks.  Proposition \ref{prop:subgraph} shows that the expected number of such cycles is $\mathrm{\Theta}(1)$; applying Markov's inequality  completes the proof.
    \end{proof}

    \begin{cor}\label{cor:shortdualenumerator}
 In a random $(\Delta_{\mathrm{B}},\Delta_{\mathrm{C}})$ LDPC code with parity-check matrix $H$ and $2<\Delta_{\mathrm{B}}<\Delta_{\mathrm{C}}$,  almost surely as $n\rightarrow \infty$, \begin{equation} \label{eq:shortdualenumerator}
       \mathcal{N}_r := |\lbrace w\in \mathrm{im}(H^{\mathsf{T}}) : |w|=r\rbrace | = \mathrm{o}\left(n^{r/2}\right)
    \end{equation}
    for any $\mathrm{O}(1)$ $r$. 
\end{cor}
\begin{proof}
    Again, Theorem \ref{thm:randomcodes} and Theorem \ref{thm:checksoundness} imply we need only worry about combinations of O(1) checks when $\Delta_{\mathrm{B}}>2$, as the code is an expander code.    We then apply Proposition \ref{prop:subgraph}.  Suppose we have a candidate subgraph of the Tanner graph containing $s$ checks in set $S_0$,  where we include every edge (and corresponding bit $\in \partial S_0$) in the subgraph.  If this set of checks gave us a word of weight $r$ in $\mathrm{im}(H^{\mathsf{T}})$, we would have at most $\frac{1}{2}(\Delta_{\mathrm{C}}s + r)$ -- this would be the case if all $r$ bits in the word have degree 1, and all other bits in the subgraph have degree 2.  Then we have \begin{equation}
        V_0 -E_0 = s + \frac{\Delta_{\mathrm{C}}s+r}{2} -\Delta_{\mathrm{C}}s = \left(1-\frac{\Delta_{\mathrm{C}}}{2}\right)s +\frac{r}{2} < \frac{r-1}{2}.
    \end{equation}
    By Markov's inequality (as $N_G\ge 0$), with probability at least $1-n^{-1/4}$, we see any candidate subgraph with a fixed number of checks $s$ at most $\mathrm{O}(n^{r/2 - 1/4})$ times.  Since there are a finite number of candidate graphs to consider, we deduce that almost surely \eqref{eq:shortdualenumerator} holds.  
\end{proof}

   \begin{cor}\label{cor:random_codeword_50}
        In the setting of Corollary \ref{cor:shortdualenumerator}, a random codeword $b\in \ker(H)$ has Hamming weight close to $n/2$: \begin{equation}
            \lim_{n\rightarrow \infty} \mathbb{P}\left( \left| |b| - \frac{n}{2}\right|>n^{1/2+\epsilon} \right) = 0. \label{eq:random_codeword_50}
        \end{equation}
    \end{cor}
    \begin{proof}
        Almost surely there are no checks of weight $\le 2$.   Use the $\pm1$-valued $\xi_i$ and evaluate \begin{equation}
            K := \sum_{i=1}^n \xi_i = n -2|b|.
        \end{equation}
        Proposition \ref{prop:xiI} shows that \begin{equation}
            \mathbb{E}\left[K^2\right] = n,
        \end{equation}
        because the only $n$ elements in $\mathrm{im}(H^{\mathsf{T}})$ are the trivial vectors arising from $(\xi_i)^2=1$.   \eqref{eq:random_codeword_50} follows from Chebyshev's inequality.
    \end{proof}

    \begin{cor}\label{cor:evendistribution}
        In a random $(\Delta_{\mathrm{B}},\Delta_{\mathrm{C}})$ LDPC code with $2<\Delta_{\mathrm{B}}<\Delta_{\mathrm{C}}$,  almost surely as $n\rightarrow \infty$, the following property holds:  if $I\in \mathrm{im}(H^{\mathsf{T}})$ has weight $\Delta_{\mathrm{C}}$, then given a codeword $x\in \ker(H)$ chosen uniformly at random, \begin{equation}
            \mathbb{P}\left[\sum_{i\in I} x_i = j \right] = \frac{1}{2^{\Delta_{\mathrm{C}}-1}}\left(\begin{array}{c} \Delta_{\mathrm{C}} \\ j \end{array} \right) \mathbb{I}(j\text{ is even}). \label{eq:evendistribution}
        \end{equation}
    \end{cor}
    \begin{proof}
        Corollary \ref{cor:usefulfora=0} shows that almost surely, the only parity-check with support in $I$ is the check $I$ itself (anything else would have weight smaller than $\Delta_{\mathrm{C}}$).  Let $\pi_I: \mathbb{F}_2^n \rightarrow \mathbb{F}_2^I$ denote the projection onto the bits in check $I$.  Notice that $\pi_I(\ker(H))$ must be a vector space, since $\pi_I$ is a linear map.  The dimension of this space must be exactly $\Delta_{\mathrm{C}}-1$ since we have exactly one unique check in this regime. Again by linearity, the distribution over random codewords must be uniform:  if \begin{equation}
            W_w := \lbrace x\in \ker(H) : \pi_I(x)=w\rbrace,
        \end{equation}
        then $\mathbb{P}(W_{w_1} )= \mathbb{P}(W_{w_1+w_2})$ since, for every $y\in W_{w_2}$, codeword $x$ and $x+y$ are equally likely.  The uniform distribution over all vectors with an even overlap with the check $I$ is given in \eqref{eq:evendistribution}.
    \end{proof}

\section{Learning random codes}\label{app:learning}
In this section, we will show that the TDenseAM architecture above learns the codewords of classical low-density parity-check (LDPC) codes efficiently. 

\subsection{How many training examples are needed?}
We now describe how the TDenseAM trains on random codewords of a high rate LDPC code with $k/n \sim \mathrm{O}(1)$.   

The first key observation is that the TDenseAMs do correctly learn the underlying code.   In other words, the ground states of $\widehat{E}$ -- after training on a large enough number of patterns $P$ -- precisely match the codewords of the input code.   Most excitingly, for any threshold -- including $a=0$ -- we do \emph{not} need to see every random codeword in order to recover all of them as local minima of $\widehat{E}$.  Bounds on a sufficient number of codewords to learn the LDPC code are provided in the following proposition:

 \begin{prop}\label{prop:patterns_to_learn}
Suppose that we train the $p$-local TDenseAM architecture with threshold $a>0$ on $P$ uniformly random codewords, from an LDPC code with $\le m$-body weights and distance $d>1$.  Then if \begin{equation}
    P > \frac{2}{a^2}\left(m\log n + \log \frac{2}{\epsilon} \right) \label{eq:patterns_to_learn}
\end{equation}  
for $0<\epsilon<1$, we find a bijection between ground states of $\widehat{E}$ and codewords in $V$, with probability at least $1-\epsilon$. 

If we consider the usual dense associative memory ($a=0$), then the lowest energy states of $\widehat{E}$ are the codewords so long as\footnote{The exponent of the polynomial scaling here is not expected to be optimal, and may depend on the energy landscape of the code itself.  The key thing is that we do not expect that $P\sim \log N$  can ever be achieved when $a=0$.} \begin{equation}
    P >  \frac{8N^{2m}}{(m-1)!^2} \left(m\log N + \log \frac{2}{\epsilon} \right)\label{eq:a=0case}
\end{equation}
\end{prop}
\noindent Note that the scaling in this equation is suboptimal, if our goal is to learn an accurate representation of the overall energy landscape.  We return to that issue in Theorem \ref{thm:a=0learn-largeA}.
\begin{proof}
    We first consider the case $a>0$.  $\widehat{E}$ is determined by the random variables \begin{equation}
        T_I :=\frac{1}{P}\sum_{\mu=1}^P\xi_I^\mu.
    \end{equation}
    The distribution of these random variables is straightforward to see:  each $T_I$ is 1 with probability 1 if $I \in \mathrm{im}(H^{\mathsf{T}})$ represents a valid parity-check (here we have a slight abuse of notation), while it is a binomial random variable (an unbiased random walk with $P$ steps, normalized by $1/P$) otherwise.  This distribution follows directly from the proof of Proposition \ref{prop:xiI}. By the assumption that $H$ was generated by at most $p$-body checks, the TDenseAM will detect every parity-check in $H$, so the remaining question is how large $P$ must be before there is a high probability that $|T_I | < a$ for each $I\notin \mathrm{im}(H^{\mathsf{T}})$. We can easily address this question using Hoeffding's inequality \eqref{eq:hoeffding}. By the union bound, the probability that we fail and find $\widehat{E} \ne E$  is bounded by (using $m\ge 2$) \begin{equation}
        \mathbb{P}[\text{fail}] < m\left(\begin{array}{c} n \\ m \end{array}\right)\cdot 2\mathrm{e}^{-Pa^2/2} < 2\mathrm{e}^{m\log n - Pa^2/2} < \epsilon.
    \end{equation} In the last step, we used  \eqref{eq:patterns_to_learn} to give a bound on $P$.

    The case of $a=0$ is similar, except that we can never truly threshold away every coupling.   We can still obtain a very conservative bound as follows: there are fewer than $N^m/(m-1)!$ possible parity-checks of weight $\le m$.  If for each $I\notin \mathrm{im}(H^{\mathsf{T}})$, we find \begin{equation}
        \widehat{T}_I < \delta_0 = \frac{(m-1)!}{2N^{m}}, \label{eq:delta_0}
    \end{equation}
    then we are guaranteed that only codewords are ground states of $\widehat{E}$.  The reason is as follows: as before, $T_I=1$ if $I\in \mathrm{im}(H^{\mathsf{T}})$, and since each bit will be contained in at least one parity-check (as the distance is larger than 1), the energy cost of an error will be positive: there will be at least one violated check, so the energy cost of any state with a violated parity-check, relative to any codeword, will be \begin{equation}
       \widehat{E}_{\text{state with error}} - \widehat{E}_{\text{codeword}} > \left(1-\frac{N^m}{(m-1)!} \delta_0 \right) - \frac{N^m}{(m-1)!} \delta_0 > 0
    \end{equation}
    using \eqref{eq:delta_0}.  So we should use Hoeffding's inequality \eqref{eq:hoeffding} and ask that each $T_I < \delta_0$, which occurs when \eqref{eq:a=0case} holds.
\end{proof}

The above proposition guarantees that after seeing enough patterns, the TDenseAM learns the underlying code.   Crucially, either \eqref{eq:patterns_to_learn} or even \eqref{eq:a=0case} are far smaller than the $\exp(\mathrm{O}(N))$ codewords of a random LDPC code.  Therefore, the associative memory is \emph{not} simply memorizing the training data, it is learning the underlying rules that generated the data set; namely, performing Hamiltonian learning.  Amazingly, in the case $a>0$, we do not even need to show the TDenseAM every linearly independent vector in $V$ before it learns the entire vector space.  The $a>0$ TDenseAM therefore truly learns to generalize on this problem, even going beyond interpolation and learning pattern $v_1+v_2$ after seeing $v_1$ and $v_2$ individually.

\subsection{Dynamics of learning: $a\rightarrow 1$}
We will first discuss how the TDenseAM learns -- in particular, the nature of the learned energy landscape $\widehat{E}$ as a function of $P$ -- in the limit $a\rightarrow 1$.   In this regime, the TDenseAM becomes identical to the SAT learning algorithm of Valiant \cite{valiant} applied to the XORSAT (LDPC) problem: we reject a constraint if we ever see it violated a single time.  Namely, $\widehat{T}_I \ne 0$ if and only if $\xi_I^\mu = \xi_I^\nu$ for every pair of training patterns $(\mu,\nu)$.  

In our discussion of the learning algorithm, it will be useful to separate out \begin{equation} \label{eq:EEjunk}
    \widehat{E}(\sigma)  = -\sum_{I\in \mathrm{im}(H^{\mathsf{T}})}\widehat{T}_I \sigma_I-\sum_{I\notin \mathrm{im}(H^{\mathsf{T}})}\widehat{T}_I \sigma_I =: E + E_{\text{junk}},
\end{equation}
where $E$ is a true code Hamiltonian that we will learn, while $E_{\text{junk}}$ is the spurious Hamiltonian that is sensitive to the specific training data we have seen.  Proposition \ref{prop:patterns_to_learn} gives us a loose bound on how fast it takes for $E_{\text{junk}}$ to become negligibly small or even zero.   However, we would like to say more about what happens to $E_{\text{junk}}$ as a function of $P$.

  There are two useful ways to think about how the TDenseAM learns.  Firstly, we can keep track of \begin{equation}
    J := \#\lbrace I : \widehat{T}_I > 0\rbrace = \sum_{I} X_I.
\end{equation}
$J$ stands for the number of junk couplings that we have not unlearned, and is a proxy for how corrupted $\widehat{E}$ is by $E_{\text{junk}}$.   For later use, we have also defined the indicator variables \begin{equation}
    X_I := \mathbb{I}[\widehat{T}_I \ne 0].
\end{equation}  The second thing we can keep track of is the rank of the learned codespace $\widehat{k} \le k$.

We will show that the way TDenseAM learns in the limit $a\rightarrow 1$ is rather different at early times depending on whether $m=2$ or $m>2$, and so we will have to tackle those problems separately.  The late time learning is very similar in both cases, so when we talk about late time behavior we will discuss both together.

\begin{prop}
   Consider the task of learning a random LDPC code with $m=2$ using the $a\rightarrow 1$ TDenseAM.   Then almost surely as $k\rightarrow \infty$, if $P< (1-\epsilon)\log_2 k$ for any $0<\epsilon<1$, the learned rank of codewords is
   \begin{equation}
        \widehat{k} = 2^{P-1}.
    \end{equation}
\end{prop}
\begin{proof}
When $m=2$, the only allowed codes that TDenseAM can learn are disjoint repetition codes: indeed if $x_i=x_j$ and $x_j=x_k$ are constraints on the codespace, by transitivity $x_i=x_k$ -- therefore, the allowed codewords correspond to independent subsets of $[n]$ which must all agree with each other.
 The rank of the code is the number of non-empty sets \begin{equation}
    \mathcal{S}(\eta_1,\ldots,\eta_P) := \lbrace i\in [n] : \xi_i^\mu =\eta_i\rbrace .
\end{equation}
For convenience, it is useful to ``gauge fix" the first $\xi_i^1 =1$.  If each $\mathcal{S}(\eta_2,\ldots,\eta_P)$ is non-empty, then the rank of the learned code must be $2^{P-1}$.

The most general code with $m=2$ corresponds to a set of repetition codes, wherein the first $k$ bits are arbitrary and the remaining $n-k$ bits are fixed.  The probability that a given $i\in [k]$ ends up in each $\mathcal{S}(\eta_2,\ldots, \eta_P)$  bucket is $2^{-(P-1)}$, and each of these are independent events.  The probability that a bucket has zero elements is \begin{equation}
    \mathbb{P}\left(\text{any } \mathcal{S}(\eta_2,\ldots, \eta_P) = \emptyset\right) \le 2^{(P-1)} \left(1-2^{-(P-1)}\right)^k.
\end{equation}
This event happens with probability $<\epsilon$ so long as \begin{equation}
    P-1 + \log_2\frac{1}{\epsilon} >k\log_2 \frac{1}{1-2^{-(P-1)}} > 2^{-(P-1)}k
\end{equation}
In the limit of large $k$, this inequality is false for any $P<(1-\epsilon)\log_2k $, for any $\epsilon>0$.
\end{proof}

\begin{prop}
    Let $0<\kappa<1$ and $k=\mathrm{\Theta}(n)$.  In the regime \begin{equation}
        P = (m-\kappa)\log_2 n,
    \end{equation}
    almost surely as $n\rightarrow \infty$, the rank of learned codewords is  \begin{equation}
        \widehat{k} = k- 2\frac{n^\kappa}{m!} (1+\mathrm{o}(1)). \label{eq:widehat_k}
    \end{equation}
\end{prop}
\begin{proof}
We begin with a useful lemma:
\begin{lem}\label{lem:linindcheck}
    Let $\mathcal{A} \subset [I]$  denote a collection of the possible learned couplings. Then, in the limit $a\rightarrow 1$, after training on $P$ random codewords, \begin{equation}
        \mathbb{P}\left[ \widehat{T}_I > 0 \text{ for all } I\in \mathcal{A}  \right] = 2^{-(P-1)A} \label{eq:PTI>0}
    \end{equation}
    where \begin{equation}
        A := \mathrm{rank}(\mathcal{A}) - \mathrm{rank}(\mathcal{A} \cap \ker(H)).  \label{eq:ArankA}
    \end{equation}
\end{lem}
\begin{proof}
    When $|\mathcal{A}|=1$, \eqref{eq:PTI>0} immediately follows from the training rule and Proposition \ref{prop:xiI}.  Firstly, if the sole $I\in\mathcal{A}$ is a valid check, then every codeword satisfies it, so $A=0$ and the probability in \eqref{eq:PTI>0} is 1.  Secondly, if $I\notin \ker (H)$, the probability that patterns $\xi_I^\mu$ with $\mu=2,\ldots, P$ match $\xi_I^1$ (which itself was $\pm 1$ with equal probability) is independently $\frac{1}{2}$ for each pattern, establishing \eqref{eq:PTI>0} as well.

    Next, suppose that $|\mathcal{A}|>1$, but $\mathcal{A}\cap \ker(H)=\emptyset$ and each check in $\mathcal{A}$ is linearly independent.  In this regime, we may apply Proposition \ref{prop:linindcheck}:   the probability that all $\widehat{T}_I$ survive training on each pattern $\mu=2,\ldots, P$ is $2^{-A}$: we must lie in an orbit that does not contain any $v_a$. 
    
    An analogous argument applies in the case where $\mathcal{A}\cap \ker(H)=0$ but $\mathcal{A}$ is not full-rank.   The values of the $\widehat{T}_I$ are compeletely determined by knowledge of their values on a linearly independent subset, where we can apply the argument of the previous paragraph.

    Lastly, suppose that $\mathcal{A}\cap \ker(H)\ne \emptyset$.  In this case, we know that every element of $\ker(H)$ in the span of $\mathcal{A}$ is never violated, so we must subtract these contributions from the count above.   If we study the reduced set $\mathcal{A}\setminus \ker(H)$, a vector space of rank $A$ as given in \eqref{eq:ArankA}, then it satisfies the previous case.
\end{proof}

    Observe that \begin{equation}
        \mathbb{E}\left[J^2\right] = \sum_{I_1} \mathbb{E}[X_{I_1}] + \sum_{I_1\ne I_2}\mathbb{E}[X_{I_1}X_{I_2}] = \sum_{I_1} \mathbb{E}[X_{I_1}] + \sum_{I_1\ne I_2}\mathbb{E}[X_{I_1}]\mathbb{E}[X_{I_2}].
    \end{equation}
    The second equality follows from Lemma \ref{lem:linindcheck}: any two inequivalent $I$ are linearly independent. Therefore, so long as the LDPC code has a subextensive fraction of all possible checks (the vast majority of possible checks are ``junk" checks), \begin{equation}
        \mathrm{Var}(J) \le \mathbb{E}[J] = (1+\mathrm{o}(1))\left(\begin{array}{c} n \\ m \end{array}\right) 2^{-(P-1)} = \frac{n^m}{m!}2^{-(P-1)} (1+\mathrm{o}(1)) = 2\frac{n^\kappa}{m!} (1+\mathrm{O}(1)). \label{eq:varJ}
    \end{equation}
 Chebyshev's inequality thus shows that \begin{equation}
        \mathbb{P}\left[ \left|\frac{J}{\mathbb{E}[J]} - 1 \right| > s \right] \le \frac{s^2}{\mathbb{E}[J]} \le \frac{s^2 m!}{2n^\kappa}
    \end{equation}
    This shows that we see $2n^\kappa /m! \times (1+\mathrm{o}(1))$ checks, almost surely as $n\rightarrow \infty$.  
    
    With high probability each bit is included in at most one surviving fake check.  This is because for any $I_1 \ne I_2$, $\mathbb{P}[X_{I_1}=1|X_{I_2}=1] = \mathbb{P}[X_{I_1}=1]$, again by the linear independence criterion in Proposition \ref{prop:linindcheck} and Lemma \ref{lem:linindcheck}.  Therefore, supposing that bit $1\in I_1$ and $X_{I_1}=1$, the probability that any other surviving check contains bit 1 is \begin{align}
        \mathbb{P}[\text{2 surviving checks contain 1} &| \text{1 surviving check}] \le \sum_{I_2\ne I_1, 1\in I_2}\mathbb{E}[X_{I_2}] \notag \\
        &= \left(\begin{array}{c} n-1 \\ m-1 \end{array}\right) 2^{-(P-1)} =\frac{2}{(m-1)! n^{1-\kappa}} \rightarrow 0.
    \end{align}
    This shows that almost every check is linearly independent as $n\rightarrow\infty$, and establishes \eqref{eq:widehat_k}.
\end{proof}

\begin{thm}\label{thm:a=1learn}
   Consider the task of learning a random LDPC code with $m>2$ and $k/n = \mathrm{O}(1)$, using the $a\rightarrow 1$ TDenseAM.   Fix any $\epsilon>0$, and suppose that \begin{equation} \label{eq:smallP_m>2}
        P < \left(\frac{m(m-1)}{2m-1} -\epsilon \right)\log_2 n.
    \end{equation}
    Then as $n\rightarrow \infty$, almost surely the learned rank of codewords is \begin{equation}
        \widehat{k} = \left\lbrace\begin{array}{ll} P-1 &\ m \text{ odd} \\ P &\ m \text{ even}  \end{array}\right.. \label{eq:hatk_m>2}
    \end{equation}
\end{thm}
\begin{proof}
    The core strategy of this proof is to argue that when \eqref{eq:smallP_m>2} is obeyed, each new observed codeword increases the rank only by 1.  If we can establish this fact, \eqref{eq:hatk_m>2} follows immediately from the fact that after being shown the first pattern $\xi_i^1$, this is the unique ground state of $\widehat{E}$ if $m$ is odd, while if $m$ is even, $\xi_i^1$ and $-\xi_i^1$ are both ground states.


  To do this, we proceed as follows.  Define \begin{equation}
      \alpha := \frac{P}{\log_2n} < \frac{m(m-1)}{2m-1}-\epsilon . \label{eq:alpha_defined}
  \end{equation}
  Lemma \ref{lem:linindcheck} shows that the probability that any spurious check has survived is given by \begin{equation}
      q := \mathbb{P}[\widehat{T}_I \ne 0 \text{ if } I\notin \mathrm{im}(H^{\mathsf{T}})] = 2n^{-\alpha},
  \end{equation}
  which is quite small.  At the same time, there are a huge number of spurious checks: $\sim n^m$, which naively already wins.   However, we need to be careful: the spurious checks can themselves be highly redundant.   
  
  We will therefore establish \eqref{eq:hatk_m>2} as follows.  First, define $K_P$ to be the set of learned checks (modulo gauging by the first seen pattern) after $P$ patterns: \begin{equation}
      K_P := \lbrace w\in \mathbb{F}_2^n : |w|=m\text{ and } w\cdot (b^\mu- b^1) = 0 \text{ for } 2\le \mu \le P\rbrace.
  \end{equation}
  Without loss of generality for what follows, we set $b^1_i=0$ to simplify notation.  
  We will show that for suitably small $\alpha$, $\dim(\mathrm{span}(K_{P+1})) - \dim(\mathrm{span}(K_{P}))=1$ with high probability.  We do this by showing that for any two $w_{1,2}\in K_{P+1}\setminus K_P$ acting on disjoint bits, $w_1+w_2\in K_{P+1}$ (w.h.p.) because it can be represented as a sum over exactly $2m$ checks that remain in $K_{P+1}$.   This latter fact will come from establishing the existence of a nice matching, which we define as follows: 

  \begin{defn}[Nice matching]\label{def:nice}
Define the graph $G_{\text{nice}}$ to be the connected subgraph of the bipartite complete graph $\mathrm{K}_{m,m}$ minus a perfect matching: see Figure \ref{fig:gnice}.   A nice matching on $K_{P}$ is a collection $M$ of $m^2+m$ bits: $(v_1,\ldots, v_{2m}; u_1,\ldots,u_{m^2-m})$ such that there exists a sequential assignment of $v_1,\ldots v_{2m}$ to the $2m$ vertices of $G_{\text{nice}}$, and $u_1,\ldots, u_{m^2-m}$ to the $m^2-m$ edges of $G_{\text{nice}}$ such that $w_1=\lbrace v_1,\ldots, v_m\rbrace$, $w_2=\lbrace v_{m+1},\ldots, v_{2m}\rbrace$, and  $K_P$ contains every candidate check made out of a vertex along with its incident edges:  \begin{equation}
         \lbrace v\rbrace \cup \bigcup_{e\in \partial v} \lbrace u_e\rbrace  \in K_P .
      \end{equation}
      We also write $\mathcal{U}(M) := \lbrace u_1,\ldots, u_{m^2-m}$, and $\mathcal{A}(M)$ as the vector space spanned by the weight-$(m-1)$ checks in $M$ with the bits $v_i$ removed.
  \end{defn}

\begin{figure}[t]
    \centering
    \includegraphics[width=0.4\linewidth]{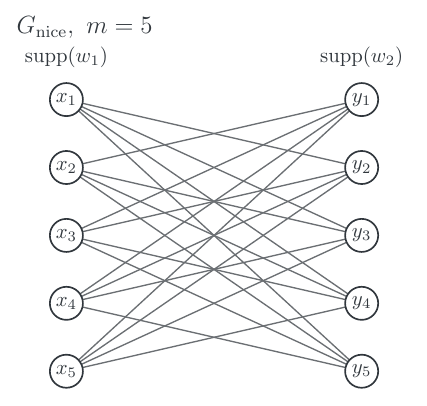}
    \caption{Illustration of $G_{\mathrm{nice}}$ in the case $m=5$.}
    \label{fig:gnice}
\end{figure}
  
\noindent We will restrict our study to candidate matchings where $u_1,\ldots, u_{m(m-1)} \in \mathcal{V}$, where $\mathcal{V} \subset [n]$ is any subset of vertices chosen to ensure the following property:
\begin{lem}\label{lem:independent-ei}
    Suppose that $k\ge Rn$ for some O(1) constant $R$, and $n$ is sufficiently large.  For any $W\subset [n]$ with $|W|=2m$, and O(1) constant $0<\rho<R$, there exists a set $\mathcal{V}\subset [n]$ with $|\mathcal{V}| = \rho n$, such that the following holds.  Let $D := \mathrm{im}(H^{\mathsf{T}})\oplus \mathrm{span}(e_i)_{i\in W}$; here $e_i$ denote the unit vectors in $\mathbb{F}_2^n$.   For every coset $[e_i] \in \mathbb{F}_2^n/D$, we have that $\lbrace [e_i]\rbrace_{i\in \mathcal{V}}$ are all linearly independent vectors.   Therefore, when sampling random codewords $b^\mu$, the random variables $b^\mu_i$ are i.i.d. Bernoulli random variables with $\mathbb{P}(b^\mu_i=1)= \frac{1}{2}$.
\end{lem}
\begin{proof}
Because we have $\mathrm{dim}(\mathrm{im}(\mathsf{H}^{\mathsf{T}}) = (1-R)n$, we know that $\mathrm{dim}(\mathbb{F}_2^n/D) \ge Rn-2m$, so any $\rho < R$ would suffice, for large enough $n$!  Since the cosets $\lbrace [e_i]\rbrace_{i\in[n]}$ span the entire space, there must be $\rho n$ of them which are linearly independent.  By Proposition \ref{prop:linindcheck}, the probability for each $\mathbb{P}(b^\mu_i=1)= \frac{1}{2}$ if $i\in\mathcal{V}$.   Moreover, even conditioned on arbitrary values of $b^\mu_j$ for $j\in W$, $b^\mu_i$ must be independent for all subsets of $i$;  if this was not true, then there would be a linear combination of $b^\mu_i$ which violates Proposition \ref{prop:linindcheck} by always taking the same value (0).   Note that the TDenseAM does not need to learn $\mathcal{V}$ -- it's role is constructive for the proof of Theorem \ref{thm:a=1learn}.
\end{proof}

 In what follows, we will condition on $w_{1,2}\cdot b^\mu = 0$ for $2\le \mu \le P$ and $w_{1,2}\cdot b^{P+1}=1$. Lemma \ref{lem:independent-ei} shows us that all $b^\mu_i$ are independent variables.  The strategy is now to show that there are typically many nice matchings in $K_{P+1}$ involving the two checks $w_{1,2}$, which would immediately imply that $w_1+w_2\in \mathrm{span}(K_{P+1})$.  To do this, we introduce the random variable \begin{equation}
          X_M := \mathbb{I}[M\text{ is a nice matching on $K_{P+1}$}].
      \end{equation}
      We further introduce independent Bernoulli random variables $\theta_M$, with $\mathbb{P}[\theta_M = 1] = \nu $ (we will choose the value of $\nu$ later).   It would suffice to argue that the following random variable is non-zero with high probability: \begin{equation}
          \mathcal{M} := \sum_M \theta_MX_M.
      \end{equation}
      We do this by invoking the following improvement over Hoeffding's inequality, discussed in \cite{suen,janson}: 
      \begin{lem}\label{lem:suen}
          Let $\lbrace X_a\rbrace_{a=1}^N$ be Bernoulli random variables:  $X_a\in\lbrace 0,1\rbrace$.  Define a graph on the $a$, with $a$ connected to $b$ by an edge ($a\in \partial b$) as infrequently as possible, while maintaining the following property:  if disjoint index sets $A$ and $B$ share no edges in common, then $\lbrace X_a\rbrace_{a\in A}$ and $\lbrace X_b\rbrace_{b\in B}$ are independent random variables.  Then \begin{equation}
              \mathbb{P}\left[\sum_{a=1}^N X_a = 0 \right] \le \exp\left(-\lambda + \Delta \mathrm{e}^{2\delta}\right)
          \end{equation}
          where \begin{subequations}
              \begin{align}
                  \lambda &= \sum_{a=1}^N \mathbb{E}[X_a], \\
                  \Delta &= \sum_{a=1}^N\sum_{b\in\partial a}\mathbb{E}[X_a X_b], \\
                  \delta &= \max_a \sum_{b\in\partial a}\mathbb{E}[X_b]. \label{eq:suen_delta}
              \end{align}
          \end{subequations}
      \end{lem}
      
      \noindent The idea is to apply this lemma to bound $\mathbb{P}[\mathcal{M}=0]$.   Clearly if $\mathcal{M}>0$, there exists a nice matching of checks in $K_{P+1}$, verifying that $w_1+w_2\in \mathrm{span}(K_{P+1})$.  We need to choose an incidence relation compatible with the requirements of Lemma \ref{lem:suen}.  A useful way to do this is to say that $M_1 \in \partial M_2$ if and only if $M_1 \ne M_2$ and $M_1$ and $M_2$ share at least one edge in common.   
      
      First, we evaluate \begin{equation}
          \lambda := \mathbb{E}[\mathcal{M}] = \frac{\lfloor \rho n\rfloor !}{(\lfloor \rho n\rfloor - m^2+m)!}\nu \mathbb{E}[X_M]
      \end{equation}
      where the first combinatorial factor represents the number of candidate nice matchings: note that the order in which we assign the coordinates does matter, which is why there is not an additional $1/(m^2-m)!$.  Lemma \ref{lem:independent-ei} implies that in the coset space $\mathbb{F}_2^n/D$, the rank of the $2m$ checks in the nice matching is $2m-1$ -- the only dependent combination is the sum of all checks (which equals $w_1+w_2$).  Using Lemma \ref{lem:linindcheck}, we then have \begin{equation}
          \mathbb{E}[X_M=1] = 2^{-P(2m-1)} = n^{-(2m-1)\alpha}. \label{eq:EXM}
      \end{equation}
       Therefore, using \eqref{eq:alpha_defined}, \begin{equation}
          \lambda \ge \nu (\rho n)^{m^2-m} n^{-(2m-1)\alpha} = \nu \rho^{m^2-m} n^{\epsilon}. \label{eq:lambda_inequality}
      \end{equation}
      On average, there are many candidate perfect matchings.  
      
      We now need to evaluate $\Delta$ and $\delta$ to show that these perfect matchings are largely independent, so the probability of finding one is extremely high. The key observation is as follows: \begin{lem}\label{lem:s_bound}
          Let $M_1$ and $M_2$ be two nice matchings ($M_1 \in \partial M_2$).  If\footnote{Recall the definition for $\mathcal{U}(M)$ at the end of Definition \ref{def:nice}.} $|\mathcal{U}(M_1)\cap \mathcal{U}(M_2)|=s$, then
          \begin{equation}
              \dim( \mathrm{row}(\mathcal{A}(M_1)) \cap \mathrm{row}(\mathcal{A}(M_2)) =: r \le \frac{2m-1}{m(m-1)}s =: a_m s.\label{eq:s_bound}
          \end{equation}   
          In other words, any two matchings with many common checks must also share a large number of bits in common.  As a consequence, \begin{equation} \label{eq:EXM1XM2}
              \mathbb{E}[X_{M_1}X_{M_2}] \le 2^{-P(2(2m-1)-r)} = \mathbb{E}[X_{M_1}]\mathbb{E}[X_{M_2}]2^{Pr}.
          \end{equation}
      \end{lem}
\begin{proof}
    Recall that a common vector between $\mathcal{A}(M_1)$ and $\mathcal{A}(M_2)$ is shared only between the coordinates in $\mathcal{U}(M_1)\cap \mathcal{U}(M_2) := \mathcal{U}_*$.  Delete the edges in $\mathcal{U}_*$ from the graphs $M_{1,2}$ (isomorphic to $G_{\mathrm{nice}}$).   The common vector must necessarily have disconnected both graphs $\mathcal{U}(M_{1,2})\setminus \mathcal{U}_*$ into $c>1$ connected components: the only way to find a trivial relation involving an unremoved edge is to sum over both checks at its endpoints: all checks on a connected component must be included with equal weight, and the sum over all checks is already a trivial constraint by construction.   We deduce that \begin{equation}
        r \le c-1.
    \end{equation}
    Now suppose that we choose a connected component of $A \subset G_{\mathrm{nice}}$ that contains $n_1$ left vertices and $n_2$ right vertices, as shown in Figure \ref{fig:gnice2}.  Without loss of generality we may take $n_1+n_2 :=n_{\mathrm{s}} \le m$.  The number of edges that must have been removed is at least \begin{equation} \label{eq:Gnice-figure}
        |\lbrace e=(v_1,v_2) : v_1\in A, v_2\notin A\rbrace| \ge (n_1+n_2)(m-1)-n_1n_2 \ge n_{\mathrm{s}}(m-1) - \frac{n_{\mathrm{s}}^2}{2} \ge \frac{m-1}{2}.
    \end{equation}
    The last inequality can be checked by brute force since $n_{\mathrm{s}}\in \lbrace 1,\ldots ,m\rbrace$.

    \begin{figure}[t]
        \centering
        \includegraphics[width=0.3\linewidth]{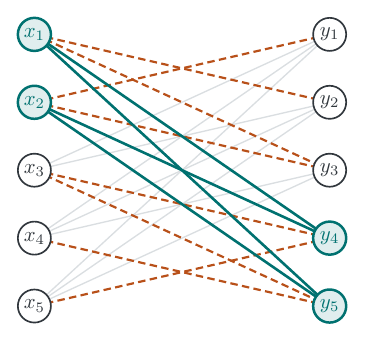}
        \caption{An illustration of a connected component $A\subset G_{\mathrm{nice}}$; the subset $A$ is highlighted in turquoise, and this example has $n_1=n_2=2$ along with $m=5$.  The red edges are the deleted edges which must have been removed in \eqref{eq:Gnice-figure}.}
        \label{fig:gnice2}
    \end{figure}

    Since $c \ge 2$ if $r>0$, we apply this cut bound to each component of $\mathcal{U}(M_1)$ post-deletion and find \begin{equation}
        s \ge c\frac{m-1}{2} \ge (m-1) \frac{m}{2m-1} (c-1),
    \end{equation}
    where the second inequality uses that $c\le 2m$.  This implies \eqref{eq:s_bound}.   \eqref{eq:EXM1XM2} follows the fact that the rank of checks in $M_1$ and $M_2$ is $2(2m-1)-r$ and from Lemma \ref{lem:linindcheck} and \eqref{eq:EXM}.
   \end{proof} 

   Let us then count the number of neighboring pairs in the adjacency graph on matchings.  Clearly, for $s>0$,  \begin{equation} \label{eq:M1M2count}
       |\lbrace (M_1,M_2) : |\mathcal{U}(M_1)\cap \mathcal{U}(M_2)|=s\rbrace| \le (m^2-m)^{2s} (\rho n)^{2m(m-1)-s},
   \end{equation}
   where the first factor counts the number of candidate identifications of two edges in the matching with the same bit, while the second factor bounds  the number of independent bits to choose between the matchings.   
   
   Now we can bound $\Delta$ as follows.  Fix $M_1$ (all matchings are the same on average), and sum over all potentially correlated $M_2$: \begin{equation}
       \Delta \le \frac{\lfloor \rho n\rfloor !}{(\lfloor \rho n\rfloor - m^2+m)!}\nu^2 \sum_{M_2 \in \partial M_1} \mathbb{E}[X_{M_1}X_{M_2}].
   \end{equation}
   This sum can be broken up by rank $r$ as follows: using \eqref{eq:M1M2count} for the case $r=0$, and using Lemma \ref{lem:s_bound} for the case $r>0$, we have \begin{align}
       \Delta &\le (\rho n)^{2m(m-1)} \nu^2 \mathbb{E}[X_{M_1}]^2\left(\sum_{s=1}^{m^2-m}\left(\frac{m^2(m-1)^2}{\rho n}\right)^s  + \sum_{r=1}^{2m-1}\sum_{s=\lceil a_m r\rceil}^{m^2-m}2^{Pr } \left(\frac{m^2(m-1)^2}{\rho n}\right)^s\right) \notag \\
       &\le C \lambda^2  n^{-\epsilon}.
   \end{align}
for sufficiently small $\epsilon$, large $n$, and a suitable constant $C$ whose exact value is unimportant.  The dominant potential contribution comes from the $r=1$ term with the smallest value of $s$.

Similarly it is easy to bound $\delta$:  a fixed matching has at most $[m(m-1)]^2(\rho n)^{m(m-1)-1}$ neighbors, so \begin{equation}
    \delta =  m^2(m-1)^2(\rho n)^{m(m-1)-1} n^{-(2m-1)\alpha} = m^2(m-1)^2 \frac{\lambda}{\rho n}.
\end{equation}

      Now let us put everything together.  We want to ensure that $2\delta < 1$ and, additionally, $\lambda > 2 \Delta \mathrm{e}^{2\delta}$, or more simply, $\lambda > 6\Delta$. This happens so long as \begin{equation}
         \lambda > 6C \lambda^2 n^{-\epsilon}. \label{eq:lambda-last}
      \end{equation}
      From \eqref{eq:lambda_inequality}, we see that \eqref{eq:lambda-last} is satisfied for large enough $n$ if we take \begin{equation}
          \nu = n^{-\epsilon/2}, 
      \end{equation}
        Lemma \ref{lem:suen} then shows that almost surely as $n\rightarrow \infty$, any two spurious and non-overlapping checks $w_{1,2}\in K_P\setminus K_{P+1}$ obey $w_1+w_2 \in \mathrm{span}(K_{P+1})$.   Indeed,  the probability of a fixed pair $w_1+w_2\notin \mathrm{span}(K_{P+1})$ falls as a stretched exponential in $n$, so the probability of this event happening for any pair still vanishes at the same rate.  By the union bound, the probability we have ever failed this test at any point is bounded by \begin{equation}
          \mathbb{P}[\text{removed more checks}] \le   n^m P \exp\left(-\lambda +\Delta \mathrm{e}^{2\delta}\right) < \exp \left(-\mathrm{\Theta}(n^{\epsilon/2})\right)
        \end{equation}
        for sufficiently large $n$.  Since we have seen that almost surely as $n\rightarrow \infty$ we never remove more than 1 check at a time at this small number of patterns, we establish \eqref{eq:hatk_m>2}.
  \end{proof}
  
\subsection{Dynamics of learning: $a=0$}
Next, we consider the opposite limit of $a=0$.  In this limit, we may expect that following \cite{krotov2016}, $\widehat{E}$ generically serves as a simple memory at small $P$ -- even though the dense associative memory is trained on random codewords, it will not uncover the rules that generated the data for some time.   This is made precise in the following discussions.

The first result is to establish an absence of stability of typical unseen perturbations until $P\sim n^{m-1}$ is sufficiently large.  More concretely:

\begin{thm}\label{thm:a=0learn-smallA}
    Consider training the $a=0$ dense associative memory on a random $(\Delta_{\mathrm{B}},\Delta_{\mathrm{C}})$ LDPC code with weights $\Delta_{\mathrm{C}} = m>\Delta_{\mathrm{B}}>2$.  Consider comparing the learned energy $\widehat{E}$ at a true codeword $c \in \ker(H)$ to the learned energy at $c+a$, where $|a|=\beta n$.   If \begin{equation}
        P = A n^{m-1}, \label{eq:Pnalpha}
    \end{equation}
    then almost surely as $n\rightarrow \infty$, \begin{equation}
        \mathbb{P}[ \widehat{E}(c) > \widehat{E}(c+a) \text{ for some }a] \rightarrow 1
    \end{equation}
    if $A < A_{\mathrm{min}}(\beta)$, for $n$-independent function $A_{\mathrm{min}}$, with $A(0)=0$. 
\end{thm}
\begin{proof}
    Without loss of generality, we may consider the stability of the codeword $c=0$ (call it $\sigma^1$ when converting to spin language), after training on $P$ random codewords $\xi_i^\mu$, since the training ensemble is translation invariant on $\ker(H)$.   It is most useful to express the dense associative memory energy function as \begin{equation}
        \widehat{E}(\sigma) = -\frac{1}{P}\sum_{\mu=1}^P \left(\sigma \cdot \xi^\mu \right)^m
    \end{equation}
    when $a=0$. We will analyze larger spin flips biased towards a single observed pattern and show that typical such choices lower the energy of $c=0$.

    Towards this end, let us consider the following setup.  Let \begin{equation}
        \Xi_1 := \lbrace i\in [n] : \xi^1_i = -1\rbrace 
    \end{equation}
   denote the sites where the first (randomly drawn) codeword differs from $\sigma^1$.  Almost surely as $n\rightarrow \infty$, Corollary \ref{cor:random_codeword_50} shows that \begin{equation}
        |\Xi_1| = (1+\mathrm{o}(1)) \frac{n}{2}. \label{eq:Xi1}
    \end{equation}    
    Let $B\subset \Xi_1$ and define \begin{equation}
        \sigma^B_i := 1-2\mathbb{I}(i\in B).
    \end{equation}
    We then consider the random variable \begin{equation}
        \mathcal{E} := \widehat{E}(\sigma_i^1) - \widehat{E}(\sigma^B_i).
    \end{equation}
    The goal will be to show that for large enough $B$, $\mathcal{E}_B>0$ almost surely when $P$ is sufficiently small.  In order to show that fact, it is useful to define \begin{equation}
        Y^\mu := \left(\sigma^B \cdot \xi^\mu \right)^m - \left(\sigma^1 \cdot \xi^\mu \right)^m \label{eq:defYmu}
    \end{equation}
    since \begin{equation}
        \mathcal{E} = \sum_{\mu=1}^P Y^\mu.
    \end{equation}
    It will also be useful to define \begin{equation}
        \beta := \frac{|B|}{n}.
    \end{equation}

    Let us first consider the case $\mu=1$.  \eqref{eq:Xi1} implies that  \begin{equation}
        Y^1 = (n-2|\Xi_1| + 2|B|)^m - (n-2|\Xi_1|)^m,
    \end{equation}
    deterministically. Assume that $\beta>0$ is a fixed O(1) constant as $n\rightarrow \infty$.  Using \eqref{eq:Xi1}, we see that \begin{equation}
        Y^1 = (1+\mathrm{o}(1)) (2\beta n)^m  \label{eq:Y1}
    \end{equation}  In contrast, for $\mu=2,\ldots, P$, $Y^\mu$ is a random variable.  Some useful moments of this random variable are captured by the following lemma: 

    \begin{lem}\label{lem:Ystats}
        Letting $E$ denote the code energy as in \eqref{eq:EEjunk}, for any $\beta>0$: \begin{subequations}
            \begin{align}
                \mathbb{E}[Y^\mu] &= E(\sigma^1) - E(\sigma^B), \label{eq:EYmu} \\
                \mathbb{E}\left[(Y^\mu - \mathbb{E}[Y^\mu])^2\right] &<   (1+\mathrm{o}(1))\cdot 2  \frac{(2m)!}{2^mm!} \left(1 - (1-2\beta)^m \right)n^m  \label{eq:EYmu2}.
            \end{align}
        \end{subequations}
    \end{lem}
    \begin{proof}
        Here it is useful to expand out \begin{equation}
            (\sigma \cdot \xi^\mu)^m= \sum_{i_1, i_2,\ldots, i_m}\xi_{i_1}^\mu\sigma_{i_1}\cdots \xi_{i_m}^\mu\sigma_{i_m}.
        \end{equation}
        Using Proposition \ref{prop:xiI}, we see that when averaging over codewords $\xi^\mu$, the only terms that survive are those where the product of $\xi$ form a valid parity-check. From Corollary \ref{cor:usefulfora=0}, we know that every valid parity-check will involve $m$ distinct bits. Using the definition of $Y^\mu$ in \eqref{eq:defYmu} then establishes \eqref{eq:EYmu}.  

        Now we turn to the second moment.  It is useful to write it as follows.   Let $J_{1,2}\subset [n]$ denote ordered $m$-element lists.  First average over random patterns $\xi$:
        \begin{equation} \label{eq:ExiYmu2}
            \mathbb{E}\left[\left(Y^\mu\right)^2\right] = \sum_{J_1,J_2: |J_{1,2}| = m} \mathbb{E}_\xi\left[\xi_{J_1}^\mu \xi_{J_2}^\mu \right] \left( \sigma^B_{J_1}-\sigma^1_{J_1}\right)\left(\sigma^B_{J_2}-\sigma^1_{J_2}\right).
        \end{equation}
        Firstly, we can again use Proposition \ref{prop:xiI} to conclude \begin{equation}
            \mathbb{E}\left[\xi_{J_1}^\mu \xi_{J_2}^\mu\right] = \mathbb{I}[J_1 \triangle J_2 \in \mathrm{im}(H^{\mathsf{T}})].
        \end{equation}
   Notice the following simple bound:  since \begin{equation}
            \sigma^B_{J_1}-\sigma^1_{J_1} = -2\mathbb{I}\left( |J_1 \cap B| \text{ odd} \right)
        \end{equation}
        is either 0 or $-2$, we have \begin{equation}
            \mathbb{E}\left[\left(Y^\mu\right)^2\right] < 4\sum_{J_1,J_2: |J_{1,2}| = m} \mathbb{I}(J_1\triangle J_2\in \mathrm{im}(H^{\mathsf{T}})) \mathbb{I}(|J_1 \cap B| \text{odd})\mathbb{I}(|J_2 \cap B| \text{odd})
        \end{equation}
         By Corollary \ref{cor:shortdualenumerator}, this sum is dominated by terms where $J_1$ and $J_2$ contain every bit an even number of times:  if they contain some collection of bits $w\in \mathbb{F}_2^n$, we see that $w\in \mathrm{im}(H^{\mathsf{T}})$ is required for the average to be non-vanishing by Proposition \ref{prop:xiI};  the contribution over all such terms is bounded by $n^{m-|w|}\mathcal{N}_{|w|} = \mathrm{o}(n^m)$ if $|w| > 0$.  In contrast the number of ordered lists $J_1$ and $J_2$ that together share an even number of each element is bounded by $2^{-m}(2m-1)!! \times  n^m $ -- the first combinatorial factor counts all the ways that $2m$ indices $i_1\cdots i_{2m}$ can be paired up; $n^m$ is the number of bits that each pair can jointly be assigned.
         Notice that in this case, we also have 
         \begin{equation}
            \mathbb{P}\left(|J_1 \cap B| \text{ odd}\right) = \mathbb{P}\left(|J_2 \cap B| \text{ odd}\right) = \sum_{j\text{ odd}}\left(\begin{array}{c} m \\ j \end{array}\right)\beta^j(1-\beta)^{m-j} = \frac{1-(1-2\beta)^m}{2}.
        \end{equation}
        Therefore using that this contribution dominates $\mathbb{E}[(Y^\mu)^2]$, we find \eqref{eq:EYmu2}.
    \end{proof}

The crucial thing we learn from Lemma \ref{lem:Ystats} is that $Y^\mu$ are not heavy-tailed random variables; we can apply the central limit theorem to $Y^2+\cdots +Y^P$.  We conclude that as $n\rightarrow \infty$ with $P = An^{m-1}$ held fixed, $\mathcal{E}$ is well approximated by a Gaussian random variable with mean and variance \begin{subequations} \label{eq:mathcalEgaussian}
    \begin{align}
        \mathbb{E}[\mathcal{E}] &\rightarrow  P\mathbb{E}\left[E(\sigma^1)- E(\sigma^B) \right] + (2\beta n)^m , \label{eq:EofE} \\
        \mathbb{E}\left[\left(\mathcal{E}-\mathbb{E}[\mathcal{E}]\right)^2\right] &= \frac{(2m)!}{2^{m-1}m!} \left(1 - (1-2\beta)^m \right) n^m P.
    \end{align}
\end{subequations}
Since $\xi^1$ is also random and has half $\pm 1$ almost surely, and by Corollary \ref{cor:evendistribution} the sampled logical $\xi^1$ is equally likely to have all allowed overlaps with a given check, we have:  \begin{align}
        \mathbb{E}\left[E(\sigma^1)- E(\sigma^B) \right] &= -2m! \mathcal{N}_m \cdot \frac{1}{2^{m-1}}\sum_{j\text{ even}}\left(\begin{array}{c} m \\ j \end{array}\right) \sum_{k \text{ odd}} \left(\begin{array}{c} j \\ k \end{array}\right) (2\beta)^k(1-2\beta)^{j-k} \notag \\
        &= -\frac{m!\mathcal{N}_m}{2^{m-2}}\sum_{j\text{ even}}\left(\begin{array}{c} m \\ j \end{array}\right) \frac{1-(1-4\beta)^{2j}}{2} \notag \\
        &= \left(-1 + \frac{1}{2^{m-1}}\frac{(2-4\beta)^{m}+(4\beta)^{m}}{2} \right)\mathcal{N}_mm! \notag \\
        &= \left(-1 + (1-2\beta)^m+(2\beta)^m\right)\mathcal{N}_m m! =:y. \label{eq:mu_def}
    \end{align}
    The overall factor of $m! \mathcal{N}_m$ counts the number of checks -- each of which itself is included $m!$ times because all possible orderings of the bits are allowed in $E$.

    Let us now combine everything together: we find that with high probability a real codeword is higher energy than our random perturbation towards a seen codeword ($\mathcal{E}>0$) if, using  \eqref{eq:Nm}, \begin{equation} \label{eq:P_A_min}
        P < \frac{1+\mathrm{o}(1)}{\Delta_{\mathrm{B}} (m-1)!} \frac{(2\beta)^m}{1-(2\beta)^m-(1-2\beta)^m}n^{m-1} =: (1+\mathrm{o}(1)) \cdot A_{\mathrm{min}}(\beta)n^{m-1}.
    \end{equation}
    Indeed, the variance of $\mathcal{E}$ scales as $n^{m-1/2}$ when $P\sim n^{m-1}$, while $\mathbb{E}[\mathcal{E}]\sim n^m$, so it suffices to check that $\mathbb{E}[\mathcal{E}] n^{-m}>0$ as $n\rightarrow \infty$, which is what \eqref{eq:P_A_min} guarantees.
\end{proof}


\begin{thm}\label{thm:a=0learn-largeA}
    In the same setting as Theorem \ref{thm:a=0learn-smallA}, if $A>A_{\mathrm{max}}(\beta)$,  for any $0<\beta<\gamma$ below the confinement scale of the code, \begin{equation}
        \mathbb{P}[ \widehat{E}(c) < \widehat{E}(c+a) \text{ for every $a$ with } |a|=\beta n] \rightarrow 1. \label{eq:basin_guarantee}
    \end{equation}
    Therefore, for sufficiently large $A$, there must be a local minimum of $\widehat{E}$ closer to the true codeword than to any other.
\end{thm}
\begin{proof}
    This proof begins identically to that of Theorem \ref{thm:a=0learn-smallA}, but we will need to find more sophisticated bounds on the distribution of $Y^\mu$.  Towards this end, for any $S\subset [n]$, define \begin{equation}
        \zeta_S := \sum_{i\in S}\zeta_i.
    \end{equation} 
    Notice that \begin{equation}
        Y^\mu := \left(\xi^\mu_{S^{\mathrm{c}}} - \xi^\mu_S\right)^m - \left(\xi^\mu_{S^{\mathrm{c}}} + \xi^\mu_S\right)^m.
    \end{equation}
    Using elementary inequalities we find \begin{equation}
        \left|Y^\mu \right| \le 2m\left|\xi^\mu_S\right| \left(\left|\xi^\mu_S\right|+\left|\xi^\mu_{S^{\mathrm{c}}}\right| \right)^{m-1} \le 2^mm\left|\xi^\mu_S\right| \left(\left|\xi^\mu_S\right|^{m-1} +\left|\xi^\mu_{S^{\mathrm{c}}}\right|^{m-1} \right)
    \end{equation}
    Now let us consider the $\mathrm{L}^q$ norms on the uniform distribution on $\ker(H)$: (slightly abusing notation between $b$ vs. $\xi$) \begin{equation}
        \lVert A(\xi)\rVert_{q} := \left[ \frac{1}{2^k}\sum_{\xi \in \ker(H)} |A(\xi)|^q \right]^{1/q}.
    \end{equation}
    H\"older's inequality then gives \begin{align}\label{eq:Y-qnorm}
        \left\lVert Y \right\rVert_q &\le 2^mm \left\lVert \xi_S \right\rVert_{mq}\left( \left( \left\lVert \xi_S \right\rVert_{mq} \right)^{m-1} + \left( \left\lVert \xi_{S^{\mathrm{c}}} \right\rVert_{mq} \right)^{m-1}\right).
    \end{align}
\begin{lem}\label{lem:Lq-norm}
    The following inequality holds for any $q\ge 2$ and some $q,n$-independent $C>0$: \begin{equation}
        \lVert \xi_S\rVert_q \le C \sqrt{q|S|}. \label{eq:Lq-norm-xiS}
    \end{equation}
\end{lem}
\begin{proof}
    This is a useful generalization of Corollary \ref{cor:random_codeword_50} (which immediately follows from this result by taking $S=[n]$).  To prove this more general result, we notice the following useful identity, similar to those used in establishing Kramers-Wannier duality \cite{Hong:2024vlr}: \begin{align}
        \mathbb{E}\left[ \mathrm{e}^{t\xi_S} \right] &= \mathbb{E}\left[ \prod_{i\in S} (\cosh t + \xi_i \sinh t)  \right] = (\cosh t)^{|S|}\sum_{\Lambda \subset S : \Lambda \in \mathrm{im}(H^{\mathsf{T}})}(\tanh t)^{|\Lambda|} \notag \\
        &< (\cosh t)^{|S|} \prod_{i \in S}\left( 1 + \sum_{\text{connected }\Lambda \subset S :  \Lambda \in \mathrm{im}(H^{\mathsf{T}}), \; i\in \Lambda}(\tanh t)^{|\Lambda|} \right) \notag \\
        &<(\cosh t)^{|S|} \prod_{i \in S}\left( 1 + \sum_{s = m}^n(\mathrm{e}(m-1)\Delta_{\mathrm{B}}\tanh t)^{\min(m,\kappa s)} \right)^{|S|}  \notag \\
        &< \exp\left[ |S| \left( \log \cosh t + \frac{(\mathrm{e}(m-1)\Delta_{\mathrm{B}}\tanh t)^m}{1-\mathrm{e}(m-1)\Delta_{\mathrm{B}}\tanh t} \right) \right] =: \mathrm{e}^{c(t)|S|},
    \end{align}
    with $c(t)<t^2$ at sufficiently small $0<t<t_0<1$ (because $m>2$).  
    In the third line we used the combinatorial fact that there are at most $(\mathrm{e}\Delta_{\mathrm{B}}(\Delta_{\mathrm{C}}-1))^s$ connected (on the bit adjacency graph) clusters in an LDPC code \cite{Hong:2024vlr}, with $m=\Delta_{\mathrm{C}}$,  along with the fact that the smallest such cluster is $m$ (Corollary \ref{cor:usefulfora=0}) and check soundness (Theorem \ref{thm:checksoundness}).  We see that \begin{equation}
        \mathbb{E}\left[ \mathrm{e}^{t\xi_S}\right] \le \mathrm{e}^{|S| \min(t ,c(t))} \le \mathrm{e}^{|S| \cdot t^2/t_0}.
    \end{equation}
    Markov's inequality gives \begin{equation}
        \mathbb{P}\left[|\xi_S| > x \right ] \le \inf_t \mathrm{e}^{|S| \cdot t^2/t_0 - xt} = \mathrm{e}^{-x^2/(2t_0|S|)}.
    \end{equation}
    \eqref{eq:Lq-norm-xiS} follows from Gaussian-like scaling.  More concretely, \begin{align}
        \lVert \xi_S\rVert_q &= \int\limits_0^\infty \mathrm{d}x \left(-\partial_x \mathbb{P}[\xi_S>x]\right) x^q = q\int\limits_0^\infty \mathrm{d}x \; x^{q-1}\mathbb{P}(\xi_S\ge x) \le q\int\limits_0^\infty \mathrm{d}x \; x^{q-1} \mathrm{e}^{-x^2/(2t_0|S|)} \notag \\
        &\le q (C_0 q\cdot 2t_0|S|)^{q/2}
    \end{align}
    for an O(1) constant $C_0$.   Since $q^{1/q}\rightarrow 1$ as $q$ becomes large, there exists a constant $C$ for which \eqref{eq:Lq-norm-xiS} holds.
\end{proof}
Now the idea is to use Lemma \ref{lem:Lq-norm} to bound the tails of the probability distribution of $\mathcal{E}$.  A second instrumental ingredient in this bound is the following ``sub-Weibull" tail bound from \cite{Kuchibhotla_2022} (we are deliberately simplifying/weakening their results, while maintaining sufficiency for our needs): \begin{thm} \label{thm:weibull}
    Let $X_1,\ldots, X_P$ denote independent and identically distributed random variables with $\mathbb{E}[X_1]=0$ whose probability distribution is ``sub-Weibull": namely obeys, for some $0<\alpha<1$, \begin{equation}
        \mathbb{P}(|X_1| > x) \le 2 \exp\left[ - \left(\frac{x}{x_0}\right)^\alpha  \right]. \label{eq:weibull}
    \end{equation}
    Then there exists a function $\mathcal{F}(\alpha)$ which is positive and finite for $\alpha>0$ such that, for any $t>0$, \begin{equation}
        \mathbb{P}\left( \left| \sum_{i=1}^P X_i \right| > \mathcal{F}(\alpha)\left( t^{1/\alpha} + \sqrt{Pt}\right)x_0 \right) \le 2\mathrm{e}^{-t} \label{eq:weibull-bound}
    \end{equation}
\end{thm}
\noindent The strategy is to apply Theorem \ref{thm:weibull} as follows: write \begin{equation}
    \widetilde{Y}(\xi) := Y(\xi)- y 
\end{equation} where $y<0$ was defined in \eqref{eq:mu_def}.  Then $\mathbb{E}[\widetilde{Y}]=0$ (averaged over codewords), so we can apply Theorem \ref{thm:weibull}.  From the triangle inequality, \begin{equation}
    \lVert \widetilde{Y} \rVert_q \le \lVert Y\rVert_q + y \le 2\lVert Y\rVert_q.
\end{equation}
Next, combining Lemma \ref{lem:Lq-norm} with \eqref{eq:Y-qnorm} and the fact that $\max(|S|,|S^{\mathrm{c}}|)\le n$, we find  \begin{equation}
    \lVert \widetilde{Y} \rVert_q \le2 \lVert Y\rVert_q \le 2m(2C)^m q^{m/2} n^{(m-1)/2}\sqrt{|S|}=: C^\prime n^{(m-1)/2}\sqrt{|S|} \cdot  q^{m/2}. \label{eq:widetildeY-bound}
\end{equation}
Then we have the following useful lemma \cite{Vladimirova_2020}: 
\begin{lem}\label{lem:weibull-2m}
    If $\lVert X_1\rVert_q \le K q^{1/\alpha} $ for all $q$, then \eqref{eq:weibull} holds with \begin{equation}
        x_0 \le \left(\frac{2\mathrm{e}}{\log 2}\right)^{1/\alpha}K.
    \end{equation} 
\end{lem}
\noindent We now combine Lemma \ref{lem:weibull-2m} with \eqref{eq:weibull-bound} and \eqref{eq:widetildeY-bound} to find, as $n\rightarrow \infty$ and with $P=An^{m-1}$, $\widetilde{Y}$ must be sub-Weibull with $\alpha=2/m$.  More precisely,  \begin{equation} \label{eq:bigthm-tail}
    \mathbb{P}\left[ \left|\sum_{\mu=2}^P Y^\mu - (P-1)y\right|> \mathcal{C} n^{(m-1)/2}\sqrt{|S|} \left(\sqrt{Pt} + t^{m/2}\right) \right] \le 2\mathrm{e}^{-t}
\end{equation}
with \begin{equation}
    \mathcal{C}:= 2\mathcal{F}\left(\frac{2}{m}\right)C^\prime.
\end{equation}
So the tail bound \eqref{eq:bigthm-tail} is strong when we set $t=\tau |S| $ (for constant $\tau$ to be determined).    Recall that the regime of interest has $|S| = \beta n$, and we need the fluctuations in \eqref{eq:bigthm-tail} to be sufficiently suppressed so that a union bound over all possible $|S|$ subsets does not find a single favorable spin flip.  In other words, we need \begin{equation}
   \mathrm{e}^{-\tau \beta n}\left(\begin{array}{c} n \\ |S| \end{array}\right) \ll 1
\end{equation} or, recalling the binary entropy function $h$ defined in \eqref{eq:binaryentropy},
\begin{equation}
    \tau \beta > h(\beta). \label{eq:tau_inequality}
\end{equation}
Note that since $h(\beta) \sim -\beta\log \beta $ as $\beta \rightarrow 0$, we can only hope to do find such a $\tau$ for a fixed positive $\beta$, which is what the theorem asks for.  

In any event, let us finalize things.  For $\tau$ obeying \eqref{eq:tau_inequality} we need  \begin{equation}
    0 > (P-1)y + \mathcal{C}n^{m/2}\sqrt{\beta}\left(\sqrt{P\tau \beta n} + (\tau \beta n)^{m/2}\right)+(2\beta n)^{m}, \label{eq:almost-done-inequality}
\end{equation}
which would ensure that the cutoff in \eqref{eq:bigthm-tail}, along with the positive energy gain from $Y^1$ given in \eqref{eq:Y1}, are not enough to push $\mathcal{E}>0$ due to the cost of violating many of the (approximately) learned checks having arose from the $\mu=2,\ldots, P$ contributions to $\widehat{E}$. From Corollary \ref{cor:linearconfinement}, \begin{equation}
    y \ge \ (1-2\delta)\Delta_{\mathrm{B}} \cdot  \beta n,
\end{equation}
 for any $\delta< 1/\Delta_{\mathrm{B}}$ and $\beta < \gamma$ for which Theorem \ref{thm:randomcodes} ensures we have a $(\gamma,\delta)$ expander code.  Almost surely as $n\rightarrow \infty$,   \begin{equation}
    (1-2\delta)\Delta_{\mathrm{B}}\beta  \cdot A > \mathcal{C}\sqrt{\beta} \left(\sqrt{Ah(\beta)} + h(\beta)^{m/2}\right) + (2\beta)^m
\end{equation}
Solve this quadratic inequality for $0<\beta<\gamma$ and call the solutions $A >A_{\mathrm{max}}(\beta)$; this establishes \eqref{eq:basin_guarantee}.
\end{proof}

Theorem \ref{thm:a=0learn-smallA} and Theorem \ref{thm:a=0learn-largeA} together imply a clear crossover when $P= \mathrm{\Theta}(n^{m-1})$ where the denseAM establishes basins of stability around each codeword, as claimed in the main text.  Let us now confirm that, as claimed in the main text, codewords are lower energy than random configurations once $P\sim n^{m-2}$, which we formalize in the proposition below: \begin{prop}
    Starting at any codeword and flipping a random set $S$ of bits towards an \emph{unseen} codeword, we find (using the notation of Theorem \ref{thm:a=0learn-smallA}) that $\widehat{E}(\sigma^B) > \widehat{E}(\sigma^1)$ with high probability as $n\rightarrow \infty$ if $P=n^{m-2+\epsilon}$ for any $\epsilon>0$.
\end{prop}
\begin{proof}
    We follow the proof of Theorem \ref{thm:a=0learn-smallA}, but if the bits we flip point us towards an unseen codeword, we simply lose the deterministic $(2\beta n)^m$ contribution from \eqref{eq:EofE}.   Then we see that \begin{equation}
        \frac{|\mathbb{E}[\mathcal{E}]|}{\sqrt{\mathrm{Var}(\mathcal{E})}} = \frac{\Delta_{\mathrm{B}}(m-1)! (1-(1-2\beta)^m-(2\beta)^m)}{\sqrt{2^{1-m}(2m)!/m!} (1-(1-2\beta)^m)} \sqrt{\frac{P}{n^{m-2}}} \sim n^{\epsilon/2}.
    \end{equation}
    Since $\mathcal{E}$ has approximately Gaussian fluctuations, $\mathrm{sign}(\mathcal{E}) < 0$ almost surely as $n\rightarrow\infty$. 
\end{proof}
Before $P\sim n^{m-2}$, we are relatively likely to see a true codeword higher energy than a random state, so it is hardly fair to say that we have ``learned" the code at this point!

It would be very interesting to study the ``mixed memories", which have been studied in dense associative memories trained on random codewords \cite{gayrard2026}, to the architecture trained on XORSAT codewords.  We leave this to future work.


\section{Learning latent variables}\label{app:latent}
With a sensible modification, the same TDenseAM architecture described above can (asymptotically) optimally recover latent variables -- namely, degrees of freedom which are not shown to the model, but which sparsify the energy $\widehat{E}$.   This is \emph{not} achievable for arbitrary problems \cite{valiant1994}.   However, random LDPC codes are a setting where latent variables are often easy to learn, and we will indeed provide some provable guarantees on their learnability in what follows.

Let us first set up the discussion at a high level.  Divide up the bits of an LDPC code into observed variables $x_1,\ldots, x_{n-\overline{n}}$ which will be given to TDenseAM as training data, and latent variables $h_{n-\overline{n}+1},\ldots, h_n$ which are not shown.  Notice that $\overline{n} < n$ denotes the number of latent bits.   The energy function, as before, is a sum over parity-checks.   We aim to generalize TDenseAM so that it can (1) identify latent variables $h$ and (2) learn a correct energy function $\widehat{E}$.   

As the LDPC problem is perfectly suited to TDenseAM, as discussed above, this will indeed be possible in many circumstances.  The strategy that we will explore is to identify latent variables by their observable shadows, which are defined as follows: 
\begin{defn}[Observable $m$-shadow]\label{def:shadow}
    A latent variable $h_\alpha $ in an LDPC code with maximal check weight $m$ has an observable $p$-shadow if there exist $p\ge 2$ subsets, any two of which $A, B \subset \lbrace 1,\ldots, n-\overline{n}\rbrace$ obey $|A|, |B| \le m-1$ and for which $\alpha \cup A, \alpha \cup B \in \mathrm{im}(H^{\mathsf{T}})$.
\end{defn}

Before continuing, let us make a few caveats.  (1) $h$ is not necessarily unique -- indeed, we'll comment more on this shortly!  (2) We do not rule out vastly more effective ways to discover latent variables than observable shadows.  Instead, we show that the observable shadow condition allows for a beautifully simple learning algorithm to discover them, and that many random codes guarantee observable shadows for latent bits, even when 
$\overline{n} = \mathrm{\Theta}(n)$.

Returning to Definition \ref{def:shadow}, the key observation is that $x_A + x_B=0$ is a parity-check involving at most $2m-2$ bits.  Therefore, we propose the following ``algorithm" which generalizes TDenseAM to the latent setting.  We should emphasize that exactly as described, this algorithm is tailor-made for XORSAT/LDPC problems (as was TDenseAM); however, there is a natural extension of the philosophy to more generic problems.

\begin{alg}[$p$-latent TDenseAM for XORSAT]\label{alg:latent}
    The $p$-latent TDenseAM architecture is a deterministic latent variable learner for XORSAT which iteratively proceeds as follows.  Let $\mathcal{V}_0 = [n-\overline{n}]$ denote the initial set of observed variables.   Given $P$ patterns, use TDenseAM to learn all parity-checks of weight $\le m$ -- collect them into a parity-check matrix $H_0$.   Now iterate the following algorithm, beginning at $t=0$: \begin{enumerate}
        \item Run TDenseAM to learn all parity-checks of weight $\le 2m-2$; collect them into parity-check matrix $H_t^\prime$.   Let \begin{equation}
            Z_t = \mathrm{span}(H_t^\prime)/\mathrm{span}(H_t)
        \end{equation}
        denote genuinely new higher-order constraints discovered.
        \item Suppose there exist $p(p-1)/2$ cosets $w_A \in \mathbb{F}_2^n/\mathrm{span}(H_t)$ such that every $w_A$ has a representative of Hamming weight $\le (m-1)$, and for every pair, $w_A+w_B \in \mathrm{span}(H_t^\prime)/\mathrm{span}(H_t)$.  Then introduce a \textbf{latent bit} $h$ such that the following parity-checks are obeyed: \begin{equation}
            h + w_A = 0. \label{eq:hwA0}
        \end{equation}
        In other words, $h$ has an observable $p$-shadow (Definition \ref{def:shadow}).
        \item If no $h$ is found, terminate the process.  Else, append $h$ to the observed variables: $\mathcal{V}_{t+1} := \mathcal{V}_t \cup \lbrace h\rbrace$.  Append learned parity-checks \eqref{eq:hwA0} to $H_t$ to make $H_{t+1}$.   Then send $t\rightarrow t+1$ and repeat.
    \end{enumerate}
\end{alg}

The 2-latent TDenseAM is an aggressive learner, but will necessarily introduce many redundant latent variables.  As a simple example, if $m=3$ and we observe $x_1+x_2+x_3+x_4=0$, the 2-latent learner does not know whether to introduce $x_1+x_2+h=x_3+x_4+h=0$ vs. $x_1+x_3+h^\prime=x_2+x_4+h^\prime=0$.   In contrast, the 3-latent TDenseAM is more cautious; as such one may expect for it to have a better performance guarantee.  Indeed, the 3-latent learning rule is analogous to those which have been used in previous studies of latent sparse learning \cite{BreslerKoehlerMoitra2019,BreslerBuhai2020}.   In the context of the XORSAT problem, we can prove that the 3-latent TDenseAM is capable of learning an extensive number of latent variables, with high probability, on a random realization of the problem.

\begin{thm}
    Consider a random $(\Delta_{\mathrm{B}},\Delta_{\mathrm{C}})$ LDPC code with $5 \le \Delta_{\mathrm{B}}<\Delta_{\mathrm{C}}$, and $6 \le \Delta_{\mathrm{C}} \le 9$ if $\Delta_{\mathrm{B}}=5$.  Decompose the bit set $B = B_{\mathrm{obs}}\sqcup B_{\mathrm{lat}}$ into observed and latent bits.  There exists $\mathrm{O}(1)$ constant $\gamma$ such that, for sufficiently large $n$, if $|B_{\mathrm{lat}}|< \gamma n$, the 3-latent TDenseAM architecture will learn every latent bit correctly (and no spurious latent bits) almost surely as $n\rightarrow \infty$, after seeing at most \begin{equation}
        P = (1+\epsilon) \frac{4(\Delta_{\mathrm{C}}-1)}{a^2}\log n \label{eq:Pthm6}
    \end{equation}random codewords as training examples, for any $\epsilon>0$.
\end{thm}

\begin{proof}
    Theorem \ref{thm:randomcodes} shows that a random LDPC code is a $(\gamma,\delta)$ expander with some O(1) $\gamma$, and any $\delta > 1/\Delta_{\mathrm{B}}$, almost surely as $n\rightarrow \infty$.   The strategy of the proof will be to assume $(\gamma,\delta)$-expansion with suitable $\gamma$ and $\delta$ (which holds with high probability), and then demonstrate that conditioned on this likely event, the strict latent TDenseAM learns a latent variable if and only if it is in $B_{\mathrm{lat}}$.  The key idea is that a certain kind of observable 3-shadow guarantees (certifies) the existence of a latent bit:

    \begin{lem}\label{lem:clean3shadow}
        Define a clean 3-shadow as a triple of observed parity-checks of size $2\Delta_{\mathrm{C}}-2$: \begin{equation}
            x_A+ x_B=x_B+x_C=x_C+x_A=0 \label{eq:3checkcondition}
        \end{equation}
        with sets $|A|=|B|=|C| = \Delta_{\mathrm{C}}-1$, all three sets disjoint, and none of the parity-checks explainable as the union of two smaller observed checks.   Then, given the code is a $(\gamma,\delta)$-expander with  \begin{equation}
            \delta < \frac{\Delta_{\mathrm{C}}+2}{6\Delta_{\mathrm{C}}}, \label{eq:delta_DeltaC}
        \end{equation}
        a clean 3-shadow certifies the existence of a unique (latent) bit $h_0$ such that the code had parity-checks \begin{equation}
            x_A+h_0=x_B+h_0=x_C+h_0=0. \label{eq:h0_pattern}
        \end{equation}
    \end{lem}
    \begin{proof}
        From Theorem \ref{thm:checksoundness}, we see that if $|H^{\mathsf{T}}s|=\mathrm{O}(1)$ then we may take $|s|=\mathrm{O}(1)$ with high probability as $n\rightarrow \infty$.  This puts us in the regime where Definition \ref{def:expansion} applies: for any $|s|<\gamma n$, the maximal number of parity-checks that can add together to make a parity-check of size $2\Delta_{\mathrm{C}}-2$ obeys: \begin{equation}
            |H^{\mathsf{T}}s| = 2\Delta_{\mathrm{C}}-2 \ge \Delta_{\mathrm{C}}(1-2\delta)|s|.
        \end{equation} 
        Clearly we must have $|s|\ge 2$, and the inequality is violated at $|s|=3$ if \eqref{eq:delta_DeltaC} is obeyed.  Hence \eqref{eq:delta_DeltaC} ensures that every parity-check of size $2\Delta_{\mathrm{C}}-2$ can be written as the sum of exactly two parity-checks.

        Now, suppose that we see the pattern \eqref{eq:3checkcondition}.  There are exactly two ways this can happen.  The first is that we see the pattern \eqref{eq:h0_pattern}, where $h_0$ may or may not be latent -- if it is not, we clearly should not introduce a new latent variable!   The other possibility is that \begin{subequations}\label{eq:ABCpattern}
            \begin{align}
                A &= A_0\sqcup \lbrace x,y\rbrace, \\
                B &= B_0\sqcup \lbrace x,z\rbrace, \\
                C &= C_0 \sqcup \lbrace y,z\rbrace. 
            \end{align}
        \end{subequations}
        Because by assumption we have seen $3(\Delta_{\mathrm{C}}-1)$ observed variables between the observed checks \eqref{eq:3checkcondition}, the pattern \eqref{eq:ABCpattern} can only be observed if all bits are already observed -- therefore there will be no latent variables.  Therefore, there is a latent variable if and only if the constraints \eqref{eq:3checkcondition} cannot be deduced from any of the observed lower-order constraints.
    \end{proof}

    It is clear that seeing a clean 3-shadow of a latent variable is a guarantee (with high probability) that we will correctly infer the intended latent variable.  Now we prove that:
    \begin{lem}
        If $|B_{\mathrm{lat}}|<\gamma n$, then almost surely as $n\rightarrow \infty$, there exists a latent variable with a clean 3-shadow.  
    \end{lem}
    Note that we will apply this lemma at each iteration of Algorithm \ref{alg:latent}, so $B_{\mathrm{lat}}$ is shrinking as we add more latent variables to the observed set. 
    \begin{proof}
        If the latent subgraph of the LDPC code was a tree, this property would follow immediately, because $\Delta_{\mathrm{B}} \ge 5$.  To prove this fact, we will need to show that the latent subgraph is ``close enough" to a tree to ensure the existence of a clean 3-shadow.  

        Let $B_{\text{share}}\subseteq B$ denote the set of bits which have the following property: \begin{equation}
            B_{\text{share}} :=\lbrace b\in B: |\partial b\cap \partial b^\prime|>1 \text{ for some }  b \ne b^\prime \in B \rbrace.  
        \end{equation} 
        We observe that $b\in B_{\text{share}}$ if and only if $b$ is contained in a subgraph of the Tanner graph isomorphic to a 4-cycle.  Furthermore, Proposition \ref{prop:subgraph} shows us that almost surely, $B_{\mathrm{share}}$ is a collection of disconnected 4-cycles in the Tanner graph.  Indeed, suppose that there was a set of 3 latent bits, all of which were in $B_{\mathrm{share}}$ and connected to each other.  This would require a finite subgraph of the Tanner graph with 2 loops, which happens with probability $\sim n^{-1}$ and can be neglected.  An analogous argument further shows that for any fixed O(1) distance $r_0$, $B_{\mathrm{share}}$ cannot contain two pairs of bits within a fixed distance $r_0$ of each other.
        
With this in mind, we now provide a proof by contradiction for the lemma. Suppose that the latent set had no clean 3-shadows.  Without loss of generality, we may assume that $B_{\mathrm{lat}}$ is connected: disconnected components cannot hinder learning any other component.   Because $B_{\mathrm{lat}}$ has no clean 3-shadows, every bit in $B_{\mathrm{lat}}\setminus B_{\text{share}}$ must have at most 2 unique-neighbor checks, while every bit in $B_{\mathrm{lat}}\cap B_{\text{share}}$ has at most 3 unique-neighbor checks.  Comparing with the definition of expansion, we have that \begin{align} \label{eq:bad_latent}
           \Delta_{\mathrm{B}}(1-2\delta) |B_{\text{lat}}| &\le  |\partial_{\mathrm{u}}B_{\text{lat}}| \le 2\left( |B_{\text{lat}}| - |B_{\text{lat}} \cap B_{\text{share}}|\right) + 3 |B_{\text{lat}} \cap B_{\text{share}}| \notag \\
           &= 2 |B_{\text{lat}}| + |B_{\text{lat}} \cap B_{\text{share}}|.
        \end{align}
        Notice that as long as we enforce \eqref{eq:delta_DeltaC} and $\Delta_{\mathrm{B}}\ge 5$ and $\Delta_{\mathrm{C}}\ge 6$,  \begin{equation}
            \Delta_{\mathrm{B}}(1-2\delta) \ge 5\left(1-\frac{8}{18}\right)=\frac{25}{9} > 2.
        \end{equation}
 Therefore, \eqref{eq:bad_latent} can only hold if \begin{equation}
     |B_{\text{lat}}| \le \frac{9}{7} |B_{\text{lat}} \cap B_{\text{share}}|.\label{eq:Blat_Bshare}
 \end{equation}

 Let us analyze \eqref{eq:Blat_Bshare} on each connected component of $B_{\mathrm{lat}}$.    Clearly a single isolated point in $B_{\mathrm{share}}$ is not possible, and such an isolated latent bit has a clean-$\Delta_{\mathrm{B}} \ge 5$-shadow.  So now consider the case where $B_{\mathrm{share}}$ is not empty.  If there is a single pair of bits in $B_{\mathrm{share}}$, then \eqref{eq:Blat_Bshare} implies that $2 \le |B_{\mathrm{lat}}| < 18/7 < 3$.   So the entire connected latent component is of shared bits, which is a contradiction: each such bit still has a clean $(\Delta_{\mathrm{B}}-2)\ge 3$-shadow.   If there are multiple disconnected components of $B_{\mathrm{share}}$ in the same connected component of $B_{\mathrm{lat}}$, then we know there is a path of length $\ge r_0$ separating them for any $r_0$.  Using this result, we deduce that the largest possible density \begin{equation}
     \frac{|B_{\mathrm{lat}}\cap B_{\mathrm{share}}|}{|B_{\mathrm{lat}}|} \le \frac{2m}{2m+(m-1)r_0} \label{eq:Blat_Bshare2}
 \end{equation}
 arises when $B_{\mathrm{lat}}$ is a tree-like subgraph with $m$ $B_{\mathrm{share}}$ pairs separated by $r_0$.  Take $r_0=2$:  if $m>1$, we find that \eqref{eq:Blat_Bshare2} contradicts \eqref{eq:Blat_Bshare}.  Therefore, in all possible cases, we have found a contradiction, completing the proof.
\end{proof}

It is now clear that the strict latent TDenseAM will provably converge by learning, one at a time, the latent bits by detecting their clean 3-shadows.  Lemma \ref{lem:clean3shadow} requires the condition \eqref{eq:delta_DeltaC} which enforces $\Delta_{\mathrm{C}}<10$ if $\Delta_{\mathrm{B}}=5$, while otherwise placing no constraints if $\Delta_{\mathrm{B}}>5$.  The 3-latent TDenseAM architecture correctly learns every visible weight-$(2\Delta_{\mathrm{C}}-2)$ constraint, by Proposition \ref{prop:patterns_to_learn}, if we train on the number of patterns given in  \eqref{eq:Pthm6}, almost surely as $n\rightarrow \infty$.  Therefore the 3-latent TDenseAM succeeds.
\end{proof}

\end{appendix}

\phantomsection
\section*{References}
\addcontentsline{toc}{section}{References}
\makeatletter
\let\bibsection\relax
\makeatother
\bibliography{thebib}
\end{document}